\documentclass[conference]{IEEEtran}

\makeatletter
\let\NAT@parse\undefined
\makeatother

\usepackage[T1]{fontenc}
\usepackage[utf8]{inputenc}
\usepackage{tabularx,ragged2e,booktabs,caption}
\newcolumntype{C}[1]{>{\Centering}m{#1}}

\usepackage{graphicx} 
\graphicspath{{figures/}}
\usepackage{epsfig} 
\usepackage{ragged2e}
\usepackage{amsmath}
\usepackage{amssymb}
\usepackage{epstopdf}
\usepackage{algpseudocode}
\usepackage[noend,ruled,linesnumbered]{algorithm2e}
\SetKwComment{Comment}{$\triangleright$\ } {}
\usepackage{multirow}
\usepackage{rotating}
\usepackage{subfigure}
\usepackage{subcaption}
\usepackage{color}
\usepackage{mysymbol}
\usepackage{url}
\usepackage{ulem}
\usepackage{enumerate}
\usepackage{stmaryrd}
\usepackage{array}
\usepackage{pifont}
\usepackage{multicol}
\usepackage[bookmarks=true]{hyperref}

\makeatletter
\renewcommand*{\@opargbegintheorem}[3]{\trivlist
  \item[\hskip \labelsep{\it\quad  #1\ #2:}] {\it(#3)}\ }
\makeatother
\newtheorem{theorem}{Theorem}[section]

\newtheorem{problem}{Problem}

\newtheorem{exmp}{Example}

\newtheorem{defn}[theorem]{Definition}
\newtheorem{definition}[theorem]{Definition}

\newcommand{\subsup}[3]{#1_{#2}^{#3}}
\newcommand{\level}[2]{\phi_{#2}^{#1}}

  \newenvironment{cexmp}[2]
  {\renewcommand{\theexmp}{\ref{#1}}\begin{exmp}{\textit{continued} (#2)}}
  {\end{exmp} \addtocounter{exmp}{-1}}

\newcommand{\state}[3]{#1_{#2}^{#3}}

\newcommand{\bhline}{\noalign{\hrule height 1.2pt}}

\newcommand{\hltl}{{\color{black}hierarchical sc-LTL}}  
\newcommand{\ltl}{{\color{black}sc-LTL}}  
\graphicspath{ {figs/} }

\usepackage[numbers,compress]{natbib}
\usepackage{cleveref}
\usepackage{amsmath, amssymb}
\usepackage{graphicx}
\usepackage{subfigure}
\usepackage{amsmath}
\usepackage{amsfonts}
\usepackage{booktabs}
\usepackage{amsmath, amsfonts, bm}
\usepackage{mathtools}
\usepackage{multirow}
\usepackage{etoolbox}
\title{Hierarchical Temporal Logic Task and Motion Planning for Multi-Robot Systems}%
\author{Zhongqi Wei$^{*,1}$, Xusheng Luo$^{*,1}$ and Changliu Liu$^{1}$\\
$^{*}$Equal contributions\\
Robotics Institute, Carnegie Mellon University, Pittsburgh, PA 15213, USA\\
\{zhongqi2, xushengl, cliu6\}@andrew.cmu.edu}
\begin{document}
\maketitle
\begin{abstract}
Task and motion planning (TAMP) for multi-robot systems, which integrates discrete task planning with continuous motion planning, remains a challenging problem in robotics. Existing TAMP approaches often struggle to scale effectively for multi-robot systems with complex specifications, leading to infeasible solutions and prolonged computation times. This work addresses the TAMP problem in multi-robot settings where tasks are specified using expressive hierarchical temporal logic and task assignments are not pre-determined. Our approach leverages the efficiency of hierarchical temporal logic specifications for task-level planning and the optimization-based graph of convex sets method for motion-level planning, integrating them within a product graph framework. At the task level, we convert hierarchical temporal logic specifications into a single graph, embedding task allocation within its edges. At the motion level, we represent the feasible motions of multiple robots through convex sets in the configuration space, guided by a sampling-based motion planner. This formulation allows us to define the TAMP problem as a shortest path search within the product graph, where efficient convex optimization techniques can be applied. We prove that our approach is both sound and complete under mild assumptions. To enhance scalability, we introduce a pruning heuristic that reduces the product graph size, enabling efficient planning for high-dimensional multi-robot systems. Additionally, we extend our framework to cooperative pick-and-place tasks involving object handovers between robots. We evaluate our method across various high-dimensional multi-robot scenarios, including simulated and real-world environments with quadrupeds, robotic arms, and automated conveyor systems. Our results show that our approach outperforms existing methods in execution time and solution optimality while effectively scaling with task complexity.

\end{abstract}
\section{Introduction}
\begin{figure}[!th]
    \centering
    \includegraphics[width=1\linewidth]{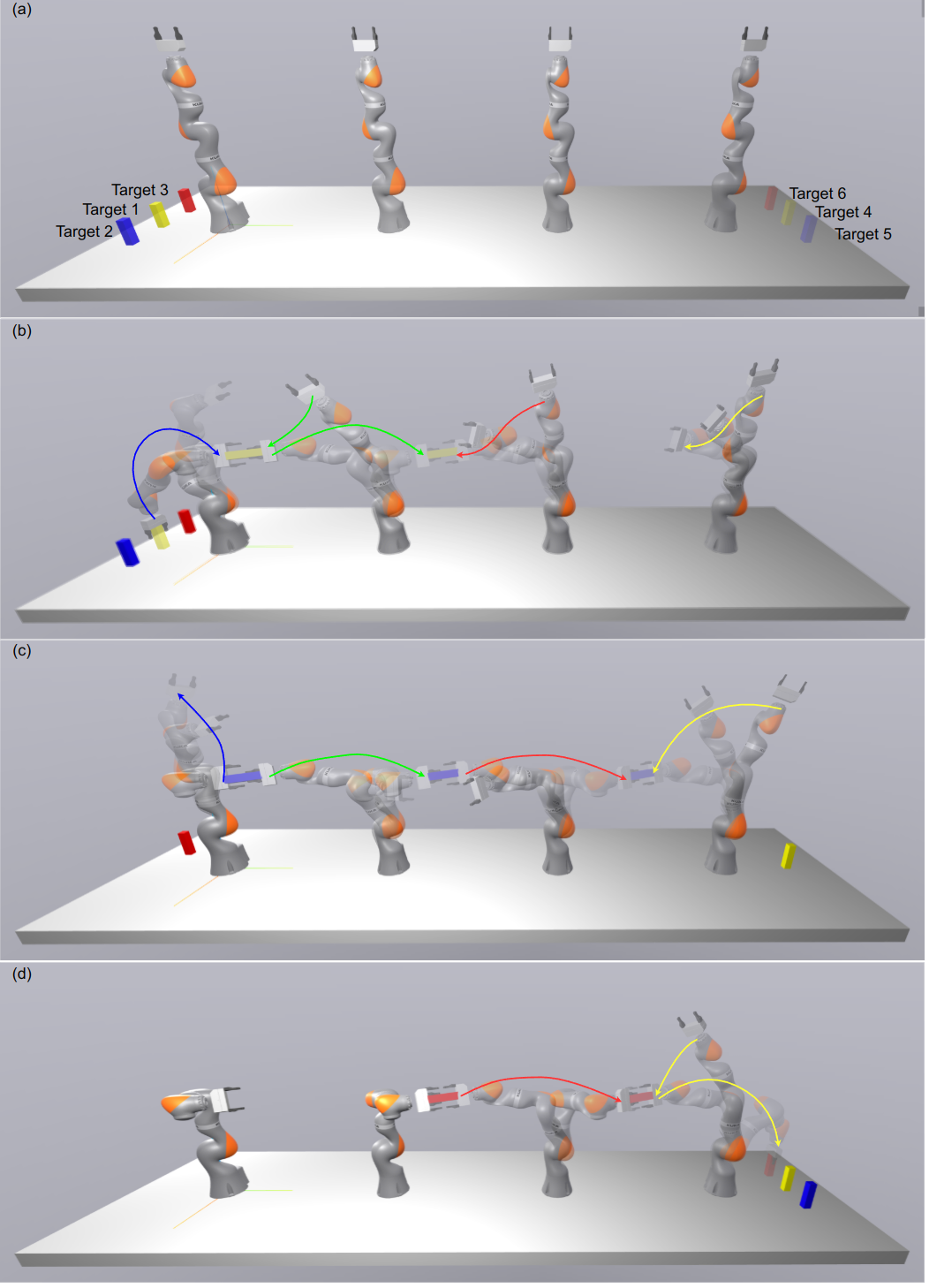}
    \caption{Given the hierarchical temporal logic specifications~\eqref{eq:four-robots HLTL}, which specify transferring the yellow, blue, and red objects in order, our approach efficiently generates collision-free trajectories for four robotic manipulators, with a total of 28 degrees of freedom, to collaboratively complete the task in the shared workspace.}
    \label{fig: four robots handover}
    \vspace{-10pt}
\end{figure}
Multi-robot systems often need to collaborate effectively and manage complex tasks. This has led to a growing demand for planning systems that can enable robots to efficiently execute long-horizon, intricate tasks. This challenge is commonly framed as a task and motion planning (TAMP) problem, formulated as a combination of discrete task planning and continuous motion planning~\cite{zhao2024survey}. TAMP is known as an NP-hard problem, particularly challenging for high-dimensional multi-robot systems tasked with complex and long-horizon operations. Traditionally, researchers in various fields have addressed these problems separately. To simplify the issues, certain assumptions are often employed, such as the existence of low-level controllers for task planning or the use of pre-defined tasks in motion planning. However, in real-world scenarios, predicting the feasibility of motion planning given a specific task specification is difficult. A feasible task description might cause an infeasible motion planning result. Therefore, traditional TAMP approaches focus on finding efficient ways to search the space of tasks~\cite{lagriffoul2014efficiently,toussaint2017multi}.


Recently, several methods have emerged that address these issues by integrating task and motion planning. The combinatorial complexity of TAMP can be partially mitigated if the problem is properly formulated~\cite{envall2023differentiable,takano2021continuous,toussaint2015logic}. However, those methods either scale poorly to complex, long-horizon tasks or struggle with local minima due to the non-convex nature of the problem. A recent study~\cite{kurtz2023temporal} introduced an innovative framework that combines Linear Temporal Logic~\cite{baier2008principles}, as an expressive specification language for long-horizon tasks, with the Graph of Convex Sets (GCS)~\cite{marcucci2023motion}, a near-optimal motion planner. This unified approach demonstrates scalability to high-dimensional systems with up to 30 degrees of freedom. However, the method is limited to single-robot scenarios, as extending it to multi-robot systems is challenging due to the NP-hard nature of task allocation. Additionally, the LTL specifications are computationally intensive. For instance, the task of sequentially collecting five keys and opening five doors in~\cite{kurtz2023temporal} required over 40 minutes, with 32 minutes dedicated solely to handling the LTL specifications, reflecting the double-exponential complexity inherent to this process. More recently, study~\cite{luo2024simultaneous} proposed a hierarchical structure for LTL specifications, significantly reducing the computational burden of handling such specifications. While their planning algorithm claims to scale to scenarios involving up to 30 mobile robots under hierarchical LTL specifications, it does not address collision avoidance and collaboration among robots.


In this work, we address the problem of integrated task and motion planning (TAMP) for multi-robot systems under hierarchical temporal logic specifications, encompassing task allocation, task planning, and motion planning. We exploit the efficiency of hierarchical temporal logic specifications at the task level and the Graph of Convex Sets (GCS) at the motion level, integrating both into a unified framework. Building on ideas from~\cite{kurtz2023temporal}, we formulate TAMP under hierarchical temporal logic specifications for multi-robot systems as a shortest-path problem within a product graph. Our proposed approach tackles several key challenges: converting hierarchical temporal logic specifications into a graph, addressing the task allocation complexities introduced by multiple robots, and efficiently connecting nodes for multiple robots in the graph of convex sets—extending beyond the existing work~\cite{marcucci2023motion} that deals with at most two robots. The shortest-path problem in the product graph is formulated as a mixed-integer convex programming (MICP) problem, which can be efficiently solved using convex relaxation techniques~\cite{marcucci2024shortest}. We theoretically prove that our approach is both sound and complete under mild assumptions, and we empirically demonstrate its efficiency through case studies involving four robotic manipulators. To manage the complexity of the product graph, we implement a pruning strategy based on the structure of the tasks. Additionally, we extend the framework to scenarios involving multiple robotic manipulators that require handovers, where the necessity for handovers is not predetermined. The primary contributions of this work are as follows:
\begin{enumerate}
   \item We formulate the multi-robot hierarchical temporal logic task and motion planning (TAMP) problem as a shortest-path problem in a product graph.
   \item At the task level, we construct a graph for hierarchical temporal logic specifications and encompass the task allocation within the edges. To construct the GCS at the motion level, capturing the dynamics of multi-robot systems, and inspired by sampling-based motion planning, we propose an approach named \textit{IRIS-RRT}, which efficiently connects the motion space.
   \item We develop a heuristic to prune the product graph, significantly improving the computational efficiency of our method.
   \item To address collaborative multi-robot pick-and-place tasks, we adapt the product graph to incorporate handover constraints and solve the shortest-path problem using mixed-integer convex programming (MICP).
   \item We provide theoretical analyses to prove the soundness and completeness of our approach under mild assumptions.
   \item We demonstrate the efficiency of our proposed method for long-horizon tasks, complex task specifications, and high-dimensional systems through several examples in both simulation and hardware. Additionally, we provide open-source code to reproduce our results.
\end{enumerate}


The remainder of this paper is organized as follows. Section \ref{Related Works} reviews related work on task and motion planning and temporal logic specifications. Section \ref{Preliminary} provides the background information on graphs of convex sets, linear temporal logic, and hierarchical specifications. The formulation of the problem and the underlying assumptions are presented in Section \ref{Problem Formulation}. Our main approach to the problem is described in Section \ref{Approach}. Section \ref{Theoretical Analysis} offers proof of the soundness and completeness of the proposed approach. The multi-robot task and motion planning examples are described in Section \ref{Experiments}. Finally, Section \ref{Limitations} discusses the limitations, and Section \ref{Conclusions} provides the conclusion of our approach.
\section{Related Works}
\label{Related Works}

\subsection{Task and Motion Planning}
{The task and motion planning (TAMP) aims to identify a sequence of symbolic actions and corresponding motion plans.} An extensive review of TAMP can be found in~\cite{garrett2021integrated,guo2023recent}. TAMP typically focuses on single-robot scenarios. In this work, we concentrate on multi-robot cases, as the aspect of task allocation is not applicable to a single robot. The primary focus in multi-robot TAMP is on the pick-up and placement of multiple objects by multiple manipulators, with the objective of determining which manipulator should pick up which objects and in what manner. One approach within this category employs search-based methods. This includes Conflict Based Search (CBS)~\cite{motes2020multi}, Monte-Carlo Tree Search (MCTS)~\cite{zhang2022mip}, search in hyper-graphs~\cite{motes2023hypergraph}, and search based on satisfiability modulo theories (SMT) solvers~\cite{pan2021general}. Another approach utilizes optimization-based methods. For example,~\cite{toussaint2015logic} proposed the logic-geometric program (LGP), which integrates continuous motion planning and discrete task specifications into optimization problems. Similarly,~\cite{envalldifferentiable} implicitly assigns actions based on the solution to a nonlinear optimization problem. Our work diverges from multi-robot TAMP in that most TAMP studies do not consider logical or temporal constraints, with only a handful addressing dependency constraints that emerge from handover operations. Our approach incorporates these constraints, adding complexity to task planning and execution.

\subsection{Control Synthesis under Temporal Logic Specifications}\label{sec:lit_mr}

Temporal logic specifications play various critical roles in control synthesis, particularly within the realm of single-robot systems. Primarily, temporal logic formulas are utilized to define task specifications. For instance,~\cite{wang2023temporal} employs temporal logic to articulate temporally extended objectives and to respond to failures during learning from demonstrations.~\cite{pacheck2023physically} implement skill repair mechanisms when existing skills are insufficient to satisfy LTL specifications. Similarly,~\cite{luo2019transfer} focuses on transferring skills across different LTL specifications. Beyond task definition, temporal logic formulas are also effective in representing constraints that dynamical systems must adhere to. For instance,~\cite{feng2024ltldog} introduces LTLDoG, a diffusion-based policy for robot navigation that complies with LTL constraints.~\cite{gu2024walking} utilizes temporal logic to define dynamic constraints, ensuring the stability of walking trajectories in bipedal robots. Additionally,~\cite{le2020single} uses temporal logic to impose constraints on switching protocols, enabling a single agent to robustly track multiple targets.

In the realm of multi-robot systems governed by LTL task specifications or constraints, LTL formulas are generally categorized into {\it local} and {\it global} forms. One strategy, as demonstrated in studies such as~\cite{guo2015multi, tumova2016multi, yu2021distributed}, involves assigning LTL tasks locally to each individual robot within the team. Alternatively, a global LTL specification can be designated for the entire team. When global LTL specifications are employed, they may either explicitly allocate tasks to specific robots~\cite{loizou2004automatic, smith2011optimal, saha2014automated, kantaros2017sampling, kantaros2018distributedOpt, kantaros2018sampling, kantaros2020stylus, kantaros2022perception, luo2019transfer, luo2021abstraction} or leave task assignments unspecified among the robots~\cite{kloetzer2011multi, shoukry2017linear, moarref2017decentralized, lacerda2019petri, chen2024real}. This latter approach aligns with the problem addressed in our current work.

Global specifications that do not explicitly assign tasks to robots typically require decomposition to facilitate task allocation. This decomposition can be achieved through three primary methods: (a) The most common technique, used in works such as~\cite{schillinger2018simultaneous,schillinger2018decomposition,faruq2018simultaneous,robinson2021multiagent,luo2024decomposition,camacho2017non,camacho2019ltl,schillinger2019hierarchical,luo2022temporal,liu2024time}, involves breaking down a global specification into multiple tasks by leveraging the transition relations within the automaton, which graphically represents an LTL formula. (b) As demonstrated in~\cite{shoukry2017linear,sahin2019multirobot}, the second approach utilizes BMC techniques~\cite{biere2006linear} to develop a Boolean Satisfaction or Integer Linear Programming (ILP) model. This model simultaneously handles task allocation and implicitly decomposes tasks within a unified framework. (c) The third method, proposed by~\cite{leahy2022fast}, interacts directly with the syntax tree of LTL formulas to divide the global specification into smaller, more manageable sub-specifications.

However, most of the aforementioned methods either adopt a hierarchical framework, where task allocation is determined first followed by low-level plan synthesis—without ensuring the feasibility at the low level—or they employ a simultaneous task allocation and planning approach to guarantee completeness, but they primarily address discrete environments and action models, which may be suitable for mobile robots but are not applicable to robotic arms. In contrast, our work distinguishes itself by performing simultaneous task allocation and planning while directly considering continuous dynamics.

\section{Preliminary}
\label{Preliminary}
{\bf Notation:} Let $\mathbb{R}$ denote the set of all real values, $[K] = \{1, \ldots, K\}$ represent the set of integers from 1 to $K$, respectively, and $|\cdot|$ denote the cardinality of a set.

\subsection{Linear Temporal Logic}
Linear Temporal Logic (LTL)~\cite{baier2008principles} is a type of formal logic whose basic ingredients are a set of atomic propositions $\pi \in \mathcal{AP}$, the boolean operators, conjunction $\wedge$ and negation $\neg$, and temporal operators, next $\bigcirc$ and until $\mathcal{U}$. LTL formulas over $\mathcal{AP}$ abide by the grammar 
\begin{align}\label{eq:grammar}
\phi::=\text{true}~|~\pi~|~\phi_1\wedge\phi_2~|~\neg\phi~|~\bigcirc\phi~|~\phi_1~\mathcal{U}~\phi_2.    
\end{align}
{For brevity, we abstain from deriving other Boolean and temporal operators, e.g., \textit{disjunction} $\vee$, \textit{implication} $\Rightarrow$, \textit{always} $\square$, \textit{eventually} $\lozenge$, which can be found in \cite{baier2008principles}.} 

An infinite \textit{word} $\sigma$ over the alphabet $2^{\mathcal{AP}}$ is defined as an infinite sequence  $\sigma=\sigma_0\sigma_1\ldots\in (2^{\mathcal{AP}})^{\omega}$, where $\omega$ denotes an infinite repetition and $\sigma_k\in2^{\mathcal{AP}}$, $\forall k\in\mathbb{N}$. The language $\texttt{Words}(\phi)=\left\{\sigma|\sigma\models\phi\right\}$ is defined as the set of words that satisfy the LTL formula $\phi$, where $\models\subseteq (2^{\mathcal{AP}})^{\omega}\times\phi$ is the satisfaction relation. In this work, we focus on a particular subset of LTL formulas known as syntactically co-safe LTL, or sc-LTL for short~\cite{kupferman2001model}. As established by~\cite{kupferman2001model}, any LTL formula encompassing only the temporal operators $\Diamond$ and $\mathcal{U}$ and written in positive normal form (where negation is exclusively before atomic propositions) is classified under syntactically co-safe formulas. Sc-LTL formulas can be satisfied by finite sequences followed by infinite repetitions. This characteristic makes sc-LTL apt for modeling and reasoning about systems with finite durations, such as those found in the robotics field. Any sc-LTL formula can be converted into a Deterministic Finite Automaton (DFA).
\begin{defn}[Deterministic Finite Automaton (DFA) ~\cite{baier2008principles}]
A DFA $\ccalA$ of a sc-LTL formula $\phi$ over $2^{\mathcal{AP}}$ is defined as a tuple $\ccalA(\phi)=\left(\ccalQ, \Sigma, \delta, q_0, \mathcal{Q}_\ccalB^F\right)$, where 
\begin{itemize}
    \item $\ccalQ$ is the set of states;
    \item $\Sigma=2^{\mathcal{AP}}$ is an alphabet;
    \item $\delta\, \subseteq\, \ccalQ\times \Sigma\times\ccalQ$ is the transition relation with $|\delta(q, \sigma)| \leq 1 $ for all states $ q \in \ccalQ $ and all symbols $ \sigma \in \Sigma $;
    \item $q_0 \in \ccalQ$ is the unique initial state;
    \item $\ccalQ^F\subseteq\ccalQ$ is a set of accepting states.
\end{itemize} 

A \textit{finite run} $\rho$ of $\ccalA$ over a finite word $\sigma=\sigma_0 \sigma_1\dots\sigma_h$ is a sequence $\rho=q_0q_1 \dots q_{h+1}$ such that $(q_{i},\sigma_i,q_{i+1})\in \delta$, $\forall i = 0, \ldots, h$. A run $\rho$ is called \textit{accepting} if $ q_{h+1}  \in  \ccalQ^F$. The words $\sigma$ that produce an accepting run of $\ccalA$ constitute the accepted language of $\ccalA$, denoted by $\ccalL_{\ccalA}$. Then~\cite{baier2008principles} proves that the accepted language of $\ccalA$ is equivalent to the words of $\phi$, i.e., $\ccalL_{\ccalA}=\texttt{Words}(\phi)$.

\end{defn}
\subsection{Hierarchical sc-LTL}
The work~\cite{luo2024simultaneous} incorporates a hierarchical structure into LTL over finite traces. In this paper, we focus on hierarchical sc-LTL, following~\cite{luo2024simultaneous}, which indicates that the approach can also be applied to sc-LTL.

\begin{defn}[Hierarchical sc-LTL~\cite{luo2024simultaneous}]\label{def:hltl}
Hierarchical sc-LTL is structured into $K$ levels, labeled $L_1, \ldots, L_K$, arranged from the highest to the lowest. Each level $L_k$ with $k \in [K]$ contains $n_k$ \ltl\ formulas. The \hltl\ specification is represented as $\Phi = \left\{\level{i}{k} \,|\, k \in [K], i \in [n_k]\right\}$, where $\level{i}{k}$ denotes the $i$-th \ltl\ formula at level $L_k$. Let $\Phi_k$ denote the set of formulas at level $L_k$, and let $\texttt{Prop}(\level{i}{k})$ represent the set of propositions appearing in formula $\level{i}{k}$. The \hltl\ follows these rules:
\begin{enumerate}
    \item \label{cond:highest} There is exactly one formula at the highest level: $n_1 = 1$.
    \item \label{cond:inclusion} Each formula at level $L_k$ consists either entirely of atomic propositions, i.e., $\texttt{Prop}(\level{i}{k}) \subseteq \ccalA\ccalP$, or entirely of formulas from the next lower level, i.e., $\texttt{Prop}(\level{i}{k}) \subseteq \Phi_{k+1}$.
    \item \label{cond:union} Each formula at level $L_{k+1}$ appears in exactly one formula at the next higher level: $\level{i}{k+1} \in \bigcup_{j \in [n_k]} \texttt{Prop}(\level{j}{k})$ and $\texttt{Prop}(\level{j_1}{k}) \cap \texttt{Prop}(\level{j_2}{k}) = \varnothing$, for $j_1, j_2 \in [n_k]$ and $j_1 \not= j_2$.
\end{enumerate}
\end{defn}

\begin{exmp}[Hierarchical sc-LTL]\label{exmp:h-sc-ltl}
The following hierarchical sc-LTL specifications state that completing tasks $a$ and $b$ and $c$, and $c$ should not be the last one to be finished:
  \begin{align}
     L_1: \quad &  \level{1}{1} =  \Diamond \level{1}{2} \wedge \neg \level{1}{2} \,\mathcal{U} \, \level{2}{2} \nonumber \\
    L_2: \quad &  \level{1}{2} = \Diamond a \wedge \Diamond b \label{eq:h-sc-ltl}\\
                & \level{2}{2} = \Diamond c \nonumber.
    \end{align}
\end{exmp}

The symbol $\level{i}{k}$ is referred to as  {\it composite proposition} if it is inside a formula, and is referred to as specification otherwise.

\begin{defn}[Specification hierarchy tree~\cite{luo2024simultaneous}]  The specification hierarchy tree, denoted as \(\ccalG_h = (\ccalV_h, \ccalE_h)\), is a tree where each node represents a specification within the \hltl, and an edge \((u, v)\) indicates that specification \(u\) contains specification \(v\) as a composite proposition. Any \hltl\ specifications can be turned into a specification hierarchy tree.
\end{defn}

A specification is termed as a leaf specification if the associated node in the graph \(\ccalG_h\) does not have any children.  Let $\Phi_{\ell}$ denote the set of leaf specifications. Note that not all leaf specifications necessarily reside at level $L_K$.

\subsection{Shorest paths in Graphs of Convex Sets (GCS)}
In this section, we introduce the shortest path formulation in GCS. It is introduced in \cite{marcucci2024shortest} and applied to robot motion planning problems in \cite{marcucci2023motion}. The goal is to find the minimum-cost path from a start vertex to a target vertex in a graph.~\cite{marcucci2024shortest} defines a Graph of Convex Sets as a directed graph $\ccalG = (\ccalV, \ccalE)$ with vertices $\ccalV$ and edges $\ccalE$. Each vertex $v \in \ccalV$ is associated with a convex set
$\ccalS_v$ and a point $s_v \in \ccalS_v$, each edge $e = (u,v) \in \ccalE$ is associated with a non-negative and convex function $l_e(s_u, s_v)$ and a convex constraint $(s_u, s_v) \in \ccalS_e$. For a fixed start vertex $s$ and target vertex $t$, we are seeking a path $p$ as a sequence of vertices that connect the start vertex $s$ and $t$ through the subset $\ccalE_p$ of the edges $\ccalE$. Denoting the set of all paths in the graph $\ccalG$ as $\ccalP$, the shortest path problem in GCS states as follows: 
\begin{equation}\label{eq:gcs}
\begin{aligned}
& {\text{min}}
& & \sum \limits_{e=(u,v) \in \ccalE_{p}}l_e(s_u, s_v)\\
& \text{s.t.}
& & p \in  \ccalP,\\
&&& s_v \in \ccalS_v, \quad \forall v \in p, \\
&&&  (s_u, s_v) \in \ccalS_e , \quad \forall e = (u,v)\in \ccalE_p.\\
\end{aligned}
\end{equation}

Although the shortest path problem (SPP) in the GCS is NP-hard, an efficient mixed-integer convex program (MICP) formulation was proposed in~\cite{marcucci2024shortest}. This MICP has a very tight convex relaxation, meaning the optimal result can be tightly approximated by the solution of the convex optimization. The GCS has been further extended to robotic motion planning around obstacles, as detailed in \cite{marcucci2023motion}. The results demonstrate that GCS is a robust trajectory optimization framework, capable of encoding various costs and constraints. 

\subsection{Convex Set (CS)-based Transition System}

For simplicity, we assume that the configuration space dimension for each robot is the same, denoted by $d$. For a multi-robot system composed of $n$ robots,  the CS-based transition system is defined as follows.

\begin{defn}[CS-based Transition System]\label{def:gcs_pts}
    A CS-based multi-robot transition system (TS) is defined as $\ccalT = (\ccalS, \Delta, S_{0}, \ccalL)$: 
    \begin{itemize}
        \item $\ccalS \subset \mathbb{R}^{nd}$ represents the set of convex sets of configuration states where all robots are guaranteed to be collision-free.
        \item $\Delta \subseteq \ccalS \times \ccalS$ is the transition relation, where $(S, S') \in \Delta$ if the sets $S$ and $S'$ are either overlapping or adjacent.
        \item $\ccalL: \ccalS \to 2^{\ccalA\ccalP}$ is the labeling function, which maps any two configuration states within the same convex set to the same sets of labels, i.e., $\ccalL(s) = \ccalL(s') = \ccalL(S) $ for all $S \in \ccalS$ and for any $s, s' \in S$. With a slight abuse of notation, we apply $\ccalL$ to a convex set and any state within the convex set.
        \item $S_{0} \in \ccalS$ is the convex set that includes the initial configuration state $(s^{0}_{1}, \ldots, s^{0}_{n})$.
    \end{itemize}
\end{defn}

Given the presence of multiple specifications in hierarchical sc-LTL, a state-specification plan associates each robot state with a sc-LTL specification that a specific robot is executing. 
\begin{defn}[State-Specification Sequence~\cite{luo2024simultaneous}]\label{defn:state_specification}
A state-specification sequence with a horizon \( h \), represented as \( \tau \), is a timed sequence \( \tau = \tau_0\tau_1\tau_2 \ldots\tau_h \). Here, \( \tau_i = ((\state{s}{1}{i}, \state{\psi}{1}{i}), (\state{s}{2}{i}, \state{\psi}{2}{i}), \ldots, (\state{s}{n}{i}, \state{\psi}{n}{i})) \) is the collective state-specification pairs of \( n \) robots at the \( i \)-th timestep, where \( (\state{s}{1}{i}, \ldots, \state{s}{n}{i} )\) is the configuration state, and \( \state{\psi}{r}{i} \in \Phi_{\ell} \cup \{\epsilon\} \), with \( \epsilon \) indicating the system's non-involvement in any leaf specification at that time.
\end{defn}

A state-specification sequence is considered to satisfy the given hierarchical sc-LTL specifications if the root specification $\level{1}{1}$ is fulfilled; we refer the reader to~\cite{luo2024simultaneous} for further details.

\section{Problem Formulation}
\label{Problem Formulation}
In this section, we introduce the problem formulation for multi-robot hierarchical temporal logic task and motion planning.
\begin{defn}[Trajectory]
    A trajectory for robot $i$ defined as $\rho_i: \mathbb{R}^+ \to \mathbb{R}^d$ is a function that maps any $t\in \mathbb{R}^+$ in time to a robot configuration $s \in \mathbb{R}^d$.
\end{defn}

\begin{problem}
\label{problem1}
Consider a $n$-robot system with initial configuration $s_0 = (s_1^0, \ldots, s_n^0)  \in \mathbb{R}^{nd}$, and hierarchical sc-LTL specifications $\Phi$, the hierarchical temporal logic task and motion planning problem requires finding the collision-free robot trajectories $\rho = [\rho_1, \ldots, \rho_n]$ with minimum-cost that satisfy the hierarchical sc-LTL specifications. The planning problem is shown as follows:
\begin{subequations}
\begin{align}
    & \underset{\rho = [\rho_1, \ldots, \rho_n]}{\text{min}}
    & & \texttt{J}(\rho) \label{eq:lt1}\\
    & \text{s.t.}
    & & \texttt{Trace} (\rho) \models \Phi, \label{eq:lt2}\\
    &&&  \rho_i(t) \cap \rho_j(t) = \emptyset,\quad \forall i, j \in [n], t \in \mathbb{R}^+, \label{eq:lt3}\\
    &&& \rho_i(t) \cap \mathcal{O} = \emptyset,\quad \forall i \in [n], t \in \mathbb{R}^+,  \label{eq:lt4}\\
    &&& \rho_i(0) = s_i^0, \quad \forall i \in [n]. \label{eq:lt5}
\end{align}
\end{subequations}
where $\mathcal{O}$ represents obstacles, and $\texttt{Trace}$ returns the trace of trajectories by applying labeling function $\ccalL$ to each state in the trajectories $\rho$.

We assume the cost $\texttt{J}$ in~\eqref{eq:lt1} is smooth and strictly convex. It can be any convex function of trajectory $\rho$. For example, the cost can be the path length and include derivatives of states. The first constraint~\eqref{eq:lt2} ensures the trajectories satisfy the task specification expressed as hierarchical sc-LTL specifications. The second constraint~\eqref{eq:lt3} and third constraint~\eqref{eq:lt4} ensure non-convex collision-avoidance constraints, requiring the robot to avoid collisions with itself and the surrounding environment. To deal with those non-convex collision-avoidance constraints, inspired by trajectory optimization~\cite{marcucci2023motion}, we mitigate those collision avoidance constraints by requiring the robot to move through a collection of safe convex sets $S_1, S_2,...\subset \mathbb{R}^{nd}$ that do not collide with obstacles. The last constraint~\eqref{eq:lt5} imposes initial conditions for each robot. 


\end{problem}
\section{Approach}
\label{Approach}
In this section, we present our approach to Problem~\ref{problem1}. The basic idea of our method involves several key steps: First, we construct the labeled convex set regions in configuration space to apply hierarchical sc-LTL specification to multi-robot motion planning. To ensure the existence of a feasible path between these convex set regions, we proposed a rapidly exploring random tree (RRT)-guided method to construct the connected convex set regions to connect those labeled convex set regions. Using these convex sets and their labels, we then construct the CS-based transition system for multi-robots, as introduced in Section~\ref{Construct the Convex Sets Region for Multi-robot System}. Second, we convert each hierarchical sc-LTL formula into a Deterministic Finite Automaton (DFA) and create a product graph by taking the product of the CS-based transition system and the DFAs, as described in Section~\ref{Construct Product Automaton}.  Note that while our approach can model robot collaboration and collision avoidance, these aspects are not considered in~\cite{luo2024simultaneous}. Next, to mitigate the computational complexity associated with the potentially large product graph, we implement a graph pruning technique to simplify the problem based on task specifications, as outlined in Section~\ref{Prune PA}. Finally, in Section~\ref{Optimization Formulation and Extension for Multi-robot Handover}, we solve the pruned product graph using the mixed integer convex program (MICP) and extend our optimization framework to handle multi-robot handover tasks. The overall architecture of our multi-robots task and motion planning algorithm is illustrated in Fig.~\ref{fig:Structure}. For a static environment, the CS-based transition system can be precomputed offline, while the remaining modules are computed online.

\begin{figure}[!t]
    \centering
    \includegraphics[width=1\linewidth]{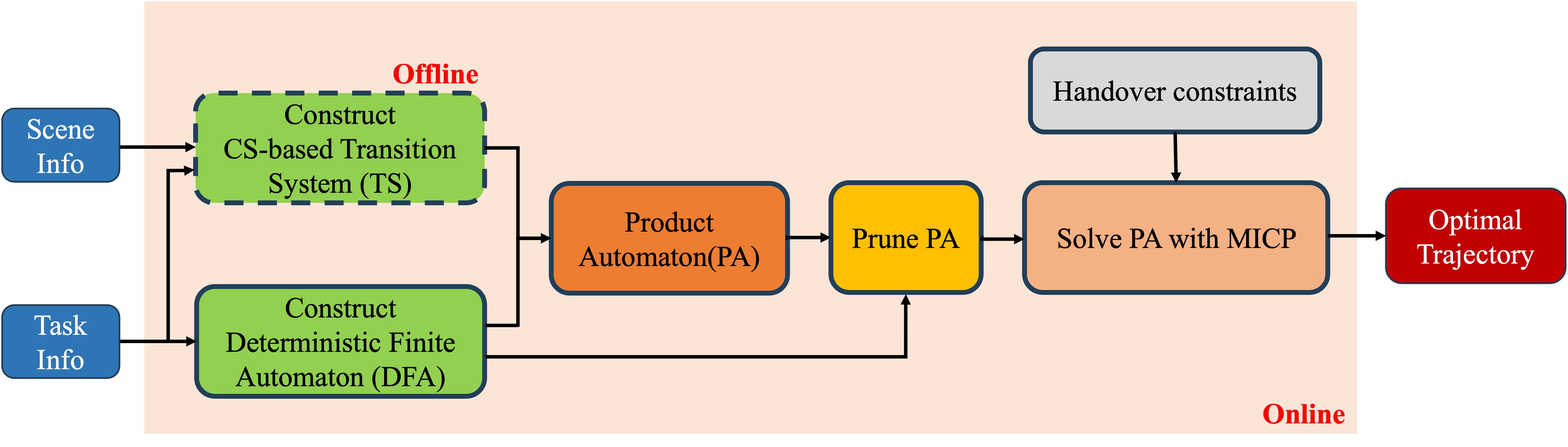}
    \caption{An architecture for hierarchical temporal logic task and motion planning, where the transition system can be precomputed offline.}
    \label{fig:Structure}
    \vspace{-10pt}
\end{figure}

\subsection{Construct CS-based Transition System for Multi-robots}
\label{Construct the Convex Sets Region for Multi-robot System}

To build a CS-based transition system defined in Def.~\ref{def:gcs_pts}, we begin by constructing labeled convex sets, represented as $S_{\text{label}}$, for each atomic proposition $\pi \in \ccalA \ccalP$, ensuring that $\ccalL (S_{\text{label}}) = \pi$. Subsequently, given any two labeled convex sets $S_{\text{label}}^i$ and $S_{\text{label}}^j$, we create a collection of {\it connected} convex sets, denoted as $\ccalS_{\text{connect}}^{i,j}$ to establish a feasible pathway between $S_{\text{label}}^i$ and $S_{\text{label}}^j$, if possible. 

\begin{algorithm}[!t]
\caption{\textbf{Construct a labeled convex set}}
\label{alg: labeled Convex Sets Region Construction}

\KwIn{Multi-robot system $\mathsf{plant}$, \\ \hspace{1cm}  atomic proposition $\pi$}
\KwOut{Labeled joint configuration $s_{\text{label}}$, labeled convex set $S_{\text{label}}$}
$s_{\text{label}} \gets \texttt{CalLabeledConfiguration}(\pi)$~\label{epa:cal_s_label}\;
$S_{\text{label}} \gets \texttt{IRIS-NP}(\mathsf{plant}, s_{\text{label}})$~\label{epa:IRIS}\;

\Return $s_{\text{label}}, S_{\text{label}}$\;
\end{algorithm}

The process of constructing labeled convex sets, as outlined in Alg.~\ref{alg: labeled Convex Sets Region Construction}, proceeds as follows: Initially, for a multi-robot system and a given atomic proposition $\pi$, we compute a configuration $s_{\text{label}}$ that satisfies the specified atomic proposition $\pi$ [line~\ref{epa:cal_s_label}]. This configuration is typically determined through robot inverse kinematics. Following this, we generate a labeled convex set $S_{\text{label}}$ starting with $s_{\text{label}}$ as the seed point, meaning that this configuration is contained within the convex set. In this work, convex sets are constructed using the IRIS-NP algorithm~\cite{petersen2023growing}, which ensures that any configuration within the convex set is free from collisions. However, IRIS-NP does not guarantee that every configuration in the convex set satisfies the atomic proposition $\pi$. To address this limitation, we incorporate additional configuration constraints in the IRIS-NP algorithm to ensure that all configurations within $S_{\text{label}}$ satisfy the atomic proposition $\pi$ [line~\ref{epa:IRIS}].  An example of this algorithm applied to a four-robot manipulator system is depicted in Fig.~\ref{fig: label_GCS}. One such additional constraint is requiring a specific robot’s end-effector to reach a designated position.




\begin{figure}[!t]
    \centering
     \subfigure[]{
      \label{fig:convex_set_a}
      \includegraphics[width=0.285\linewidth]{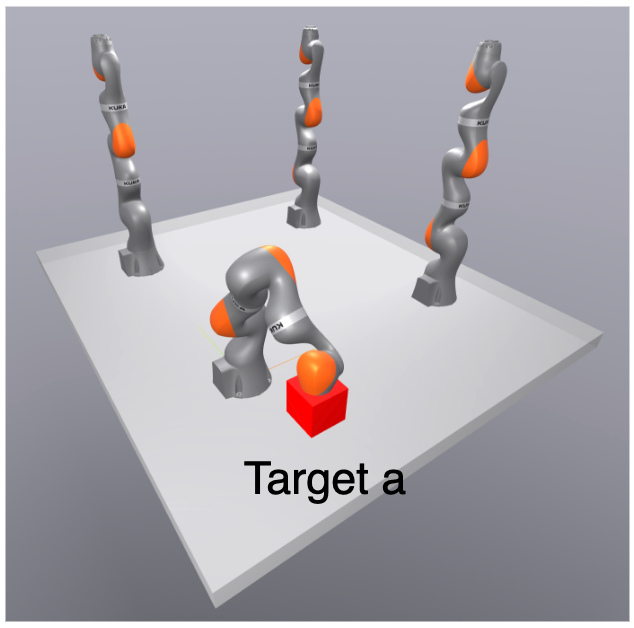}}
     \subfigure[]{
      \label{fig:convex_set_b}
      \includegraphics[width=0.3\linewidth]{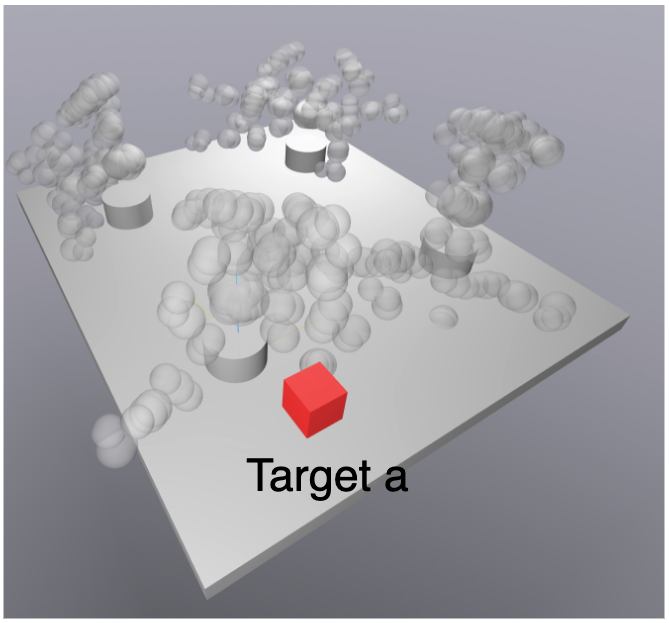}}
     \subfigure[]{
      \label{fig:convex_set_c}
      \includegraphics[width=0.295\linewidth]{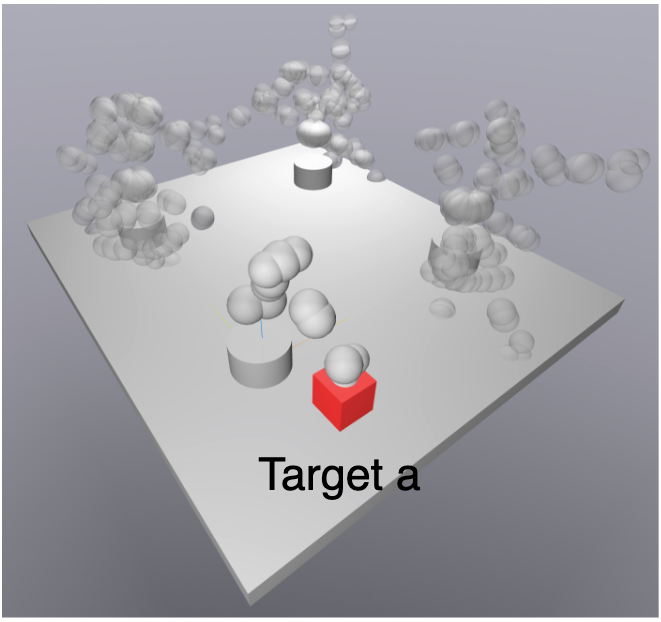}}
    \caption{An example of generating the labeled convex region for a multi-robot system. The atomic proposition is $\mathsf{target\ a}$. In Fig.~\ref{fig:convex_set_a}, the atomic proposition $\pi$ is used to compute the labeled joint configuration $s_{\text{label}}$ through robot inverse kinematics. In this configuration, the bottom robot reaches the position labeled as $\mathsf{target\ a}$, while the configurations of the remaining robots are unconstrained. Using $s_{\text{label}}$ as a seed point, the IRIS-NP algorithm can generate a convex set region that contains this seed point. The sampled configurations inside the convex region are shown in Fig.~\ref{fig:convex_set_b}. Note that not all configurations of the bottom robot ensure reaching $\mathsf{target\ a}$. Fig.~\ref{fig:convex_set_c} illustrates the labeled convex region $S_{\text{label}}$ generated by adding the bottom robot's end-effector position constraints in IRIS-NP algorithm, ensuring that all configurations in the labeled convex region satisfy the atomic proposition $\mathsf{target\ a}$.}
    \label{fig: label_GCS}
    \vspace{-10pt}
\end{figure}

\begin{algorithm}[!t]
\caption{\textbf{IRIS-RRT}}
\label{alg: RRT based Convex Sets Region Construction}

\KwIn{Multi-robot system $\mathsf{plant}$, start configuration $s_{\text{label}}^i$ and target configuration $s_{\text{label}}^j$}
\KwOut{Sets of connected convex sets $\ccalS_{\text{connect}}^{i,j}$}

$\textsf{path} \gets \texttt{RRT}(s_{\text{label}}^i, s_{\text{label}}^j)$~\label{epa:rrt_path}\; 
$\ccalS_{\text{connect}}^{i,j} \gets \texttt{IRIS-NP}(\textsf{plant}, s_{\text{label}}^i)$~\label{epa:convex_region_start}\;
$s_{\text{seed}}^{\text{old}} \gets s_{\text{label}}^i$\;

\While{$s_{\text{label}}^j$ \textbf{is not in convex sets} $\ccalS_{\text{connect}}^{i,j}$\label{epa:for}}{
    $\mathop{\textbf{maximize}}\limits_{s_{\text{seed}}^{\text{new}}}$  \texttt{distance\_along\_path}$(s_{\text{seed}}^{\text{old}}, s_{\text{seed}}^{\text{new}})$~\label{epa:optimization1}\;
    \textbf{subject to} $s_{\text{seed}}^{\text{new}} \in \ccalS_{\text{connect}}^{i,j}$ and $s_{\text{seed}}^{\text{new}} \in \mathsf{path}$~\label{epa:optimization2}\;
    $S_{\text{iris}} \gets \texttt{IRIS-NP}(\textsf{plant}, s_{\text{seed}}^{\text{new}})$~\label{epa:construct_connect1}\;
    $\ccalS_{\text{connect}}^{i,j} \gets \ccalS_{\text{connect}}^{i,j} \cup \{S_{\text{iris}}\}$~\label{epa:construct_connect2}\;
    $s_{\text{seed}}^{\text{old}} \gets s_{\text{seed}}^{\text{new}}$~\label{epa:construct_connect3}\;
}

\Return $\ccalS_{\text{connect}}^{i,j}$\;
\end{algorithm}

After generating the labeled convex sets $\ccalS$, we construct a sequence of connected convex sets $\ccalS_{\text{connect}}^{i,j}$ to ensure connectivity between any two labeled convex sets, if possible. For high-dimensional degree-of-freedom (DoF) multi-robot systems, generating these connected convex regions randomly in the robot configuration space might not successfully establish connections between the labeled convex sets, and the optimal solution might not traverse these convex sets. To address this challenge, we introduce the IRIS-RRT algorithm, outlined in Alg.~\ref{alg: RRT based Convex Sets Region Construction}. Initially, Alg.~\ref{alg: RRT based Convex Sets Region Construction} uses Rapidly-exploring Random Tree (RRT) method~\cite{lavalle2001rapidly} to find a feasible path between the two labeled configurations [line~\ref{epa:rrt_path}]. Subsequently, it constructs a set of connected convex sets along this path, spacing the seed configuration states at maximum intervals along the feasible path to minimize the number of connected convex sets [lines~\ref{epa:for}-\ref{epa:construct_connect3}]. To this end, starting with the initial configuration state $s_{\text{label}}^i$, its convex set $S_{\text{label}}^i$ is used to initialize $\ccalS_{\text{connect}}^{i,j}$. The next step involves identifying another configuration state along the feasible path that is the farthest from $s_{\text{label}}^i$ yet still within $\ccalS_{\text{connect}}^{i,j}$. This configuration state then serves as a new seed, and the process is iteratively continued to extend $\ccalS_{\text{connect}}^{i,j}$ until it contains the goal configuration $s_{\text{label}}^j$. Note that the RRT method in Alg.~\ref{alg: RRT based Convex Sets Region Construction} could be replaced by any motion planning algorithm. A 2D example of the construction of convex sets by IRIS-RRT is shown in Fig.~\ref{fig: 2D example for convex region construction}.

\begin{figure}[!t]
    \centering
    \includegraphics[width=1\linewidth]{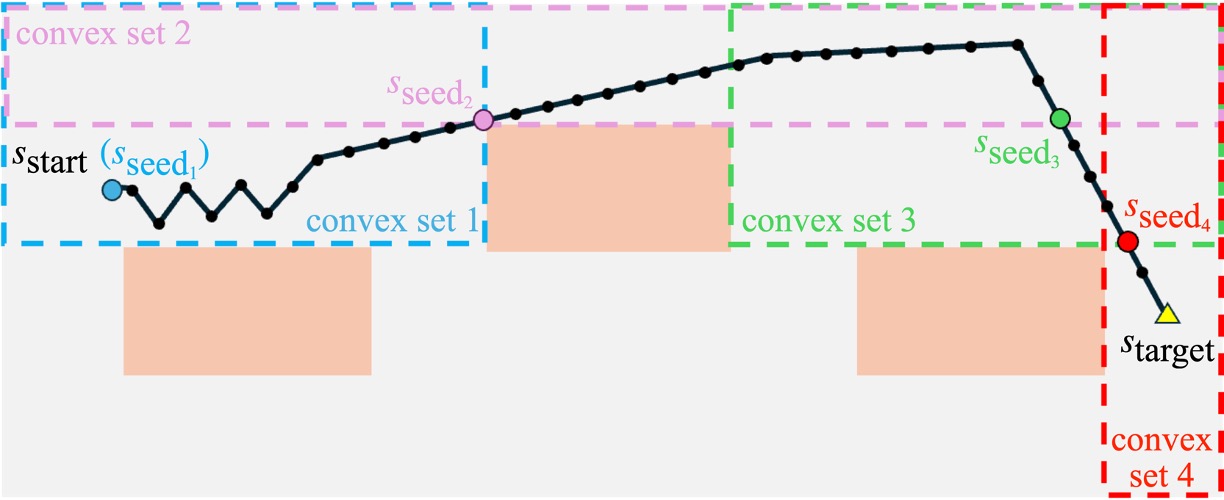}
    \caption{A 2D example for collision-free convex set construction. The orange blocks represent obstacles.  Given the robot's start configuration $s_{\text{start}}$ and goal configuration $s_{\text{target}}$, IRIS-RRT algorithm uses RRT path (black line) as a guide to approximate the connected convex sets to connect the start and goal configurations. In the graph, $s_{\text{seed}_i}$ represents the seed configuration anchoring the $i$-th convex set, which is depicted as a rectangle surrounded by dashed lines. Each pair of seed configuration and its corresponding convex set is highlighted in the same color. Note that the seed configurations typically lie at the intersections of the path generated by RRT and the previous convex sets. Moreover, the path generated by RRT does not need to be smooth or optimal, as it primarily serves to guide the construction of convex sets.}
    \label{fig: 2D example for convex region construction}
    \vspace{-10pt}
\end{figure}



\subsection{Construct Product Automaton}
\label{Construct Product Automaton}

In what follows, let $\delta(q, q')$ denote the propositional logic formula  that enables the transition from $q$ to $q'$ in DFA. We begin by addressing the challenge of constructing a graph, referred to as the {\it total product DFA}, for hierarchical sc-LTL specifications. First, for a set of specifications that have the same level, we construct their product DFA.

\begin{defn}[Product DFA (PDFA)]
Consider a DFA \(\subsup{\ccalA}{k}{i} = (\subsup{\ccalQ}{k}{i}, \subsup{\Sigma}{k}{i}, \subsup{\delta}{k}{i}, \subsup{q}{0,k}{i}, \subsup{\ccalQ}{k}{F,i})\) of the $i$-th specification $\level{i}{k}$ at level $k$. The PDFA for level $k$, denoted as $\ccalA_k = (\ccalQ_k, \Sigma_k, \delta_k, q_{0,k}, \ccalQ^F_k)$, is defined as follows:
\begin{itemize}
\item $\ccalQ_k = \subsup{\ccalQ}{k}{1} \times \ldots \times \subsup{\ccalQ}{k}{n_k}$ is the Cartesian product of automaton states across the specifications at level $L_k$;
\item $\Sigma_k = \subsup{\Sigma}{k}{1} \times \ldots \times \subsup{\Sigma}{k}{n_k}$ is the combined set of symbols from all DFAs at this level, where $\subsup{\Sigma}{k}{i} = 2^{\texttt{Prop}(\level{i}{k})}$ if $\level{i}{k}$ is a non-leaf specification, representing the child specifications of $\level{i}{k}$ at the immediate lower level $k+1$; otherwise, $\subsup{\Sigma}{k}{i} = 2^{\ccalA\ccalP}$;
\item $\delta_k \subseteq \ccalQ_k \times \Sigma_k \times \ccalQ_k$ is the transition relation, where a transition $\left((q_k^1, \ldots, \subsup{q}{k}{n_k}), (\subsup{\sigma}{k}{1}, \ldots, \subsup{\sigma}{k}{n_k}), (\subsup{q}{k}{1’}, \ldots, \subsup{q}{k}{n_k’})\right) \in \delta_k$ exists if $(\subsup{q}{k}{i}, \subsup{\sigma}{k}{i}, \subsup{q}{k}{i’}) \in \subsup{\delta}{k}{i}$ for all $i \in [n_k]$;
\item $q_{0,k} = (\subsup{q}{0,k}{1}, \ldots, \subsup{q}{0,k}{n_k})$ is the initial state of the product automaton;
\item $\ccalQ^F_k = \subsup{\ccalQ}{k}{F,1} \times \ldots \times \subsup{\ccalQ}{k}{F,n_k}$ is the set of  accepting states.
\end{itemize}
\end{defn}

Note that, by designing $\Sigma_k$ to be the product of individual symbols, each DFA is determined separately as to whether a transition occurs. 

\begin{defn}[Total PDFA (TPDFA)]~\label{def:total_pdfa}
The TPDFA $\ccalA = (\ccalQ, \Sigma, \delta, q_0, \ccalQ^F)$ for hierarchical sc-LTL specifications is detailed as follows:
\begin{itemize}
\item $\ccalQ = \ccalQ_K \times \ldots \times \ccalQ_1$ represents the set of product states across all levels;
\item $\Sigma = \Sigma_K \times \ldots \times \Sigma_1$ denotes the set of  symbols;
\item $\delta \subseteq \ccalQ \times \Sigma \times \ccalQ$ defines the transition relation, where a transition $((q_K, \ldots, q_1), (\sigma_K, \ldots, \sigma_1), (q’_K, \ldots, q’_1)) \in \delta$ is valid if:
\begin{itemize}
\item $(q_k, \sigma_k, q’_k) \in \delta_k$ for each $k \in [K]$;
\item $\sigma_k = (\subsup{\sigma}{k}{1}, \ldots, \subsup{\sigma}{k}{n_k})$ for all $k \in [K]$, with $\subsup{\sigma}{k}{i} = \{\subsup{\phi}{k+1}{j} \in \texttt{Prop}(\level{i}{k}) \,|\, \subsup{q}{k+1}{j} \in \subsup{\ccalQ}{k+1}{F,j}\}$ if $\level{i}{k}$ is a non-leaf specification, representing the child specifications of $\level{i}{k}$ at the immediate lower level $k+1$ that are fulfilled given $q_{k+1}$; otherwise, $\subsup{\sigma}{k}{i} \in \subsup{\Sigma}{k}{i} = 2^{\ccalA\ccalP}$.
\end{itemize}
\item $q_0 = (q_{0,K}, \ldots, q_{0,1})$ is the initial product state;
\item $\ccalQ^F = \{ q \in \ccalQ \, | \, \subsup{q}{1}{1} \in \subsup{\ccalQ}{1}{F,1}\}$ is the set of accepting product states where the root specification $\level{1}{1}$ is satisfied.
\end{itemize}
\end{defn}

The transition relation in Def.~\ref{def:total_pdfa} is constructed iteratively, starting from the bottom level upwards. For leaf specifications, transitions are determined based on atomic propositions, whereas for non-leaf specifications, transitions are defined by the truth of composite propositions from the immediately lower level.

 \begin{figure}[!t]
    \centering
     \subfigure[$\level{1}{1}$]{
      \label{fig:p0}
      \includegraphics[width=0.35\linewidth, trim={0.6cm 0cm 0.2cm 0cm}, clip]{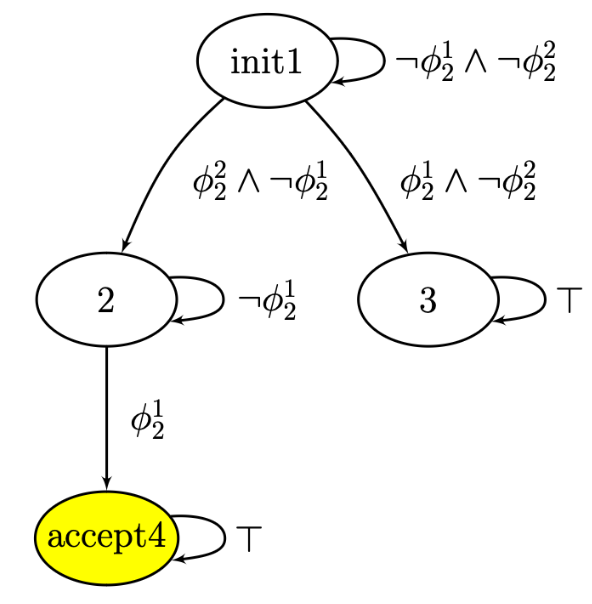}}
     \subfigure[$\level{1}{2}$]{
      \label{fig:p100}
      \includegraphics[width=0.35\linewidth, trim={1cm 0cm 0.2cm 0cm}, clip]{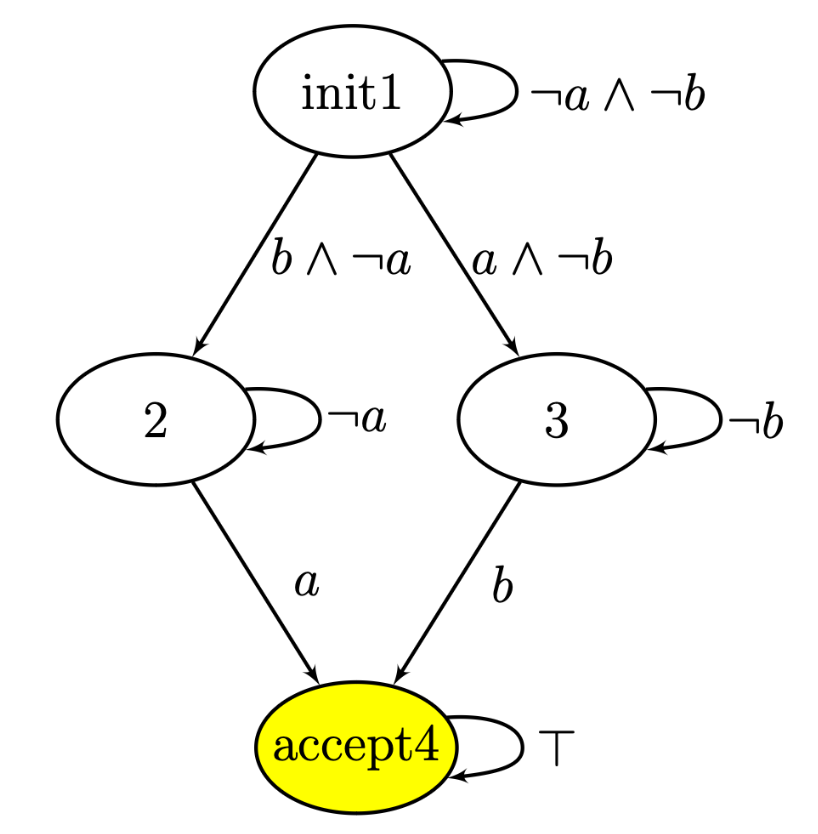}}
    \subfigure[$\level{2}{2}$]{
  \label{fig:p200}
  \includegraphics[width=0.18\linewidth, trim={0.3cm 0cm 0.2cm 0cm}, clip]{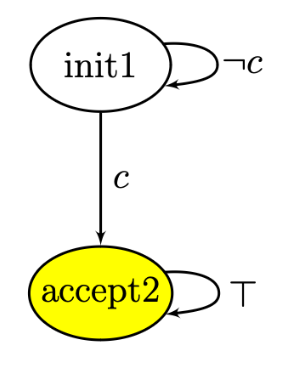}}
 \subfigure[TPDFA]{
  \label{fig:product}
  \includegraphics[width=\linewidth]{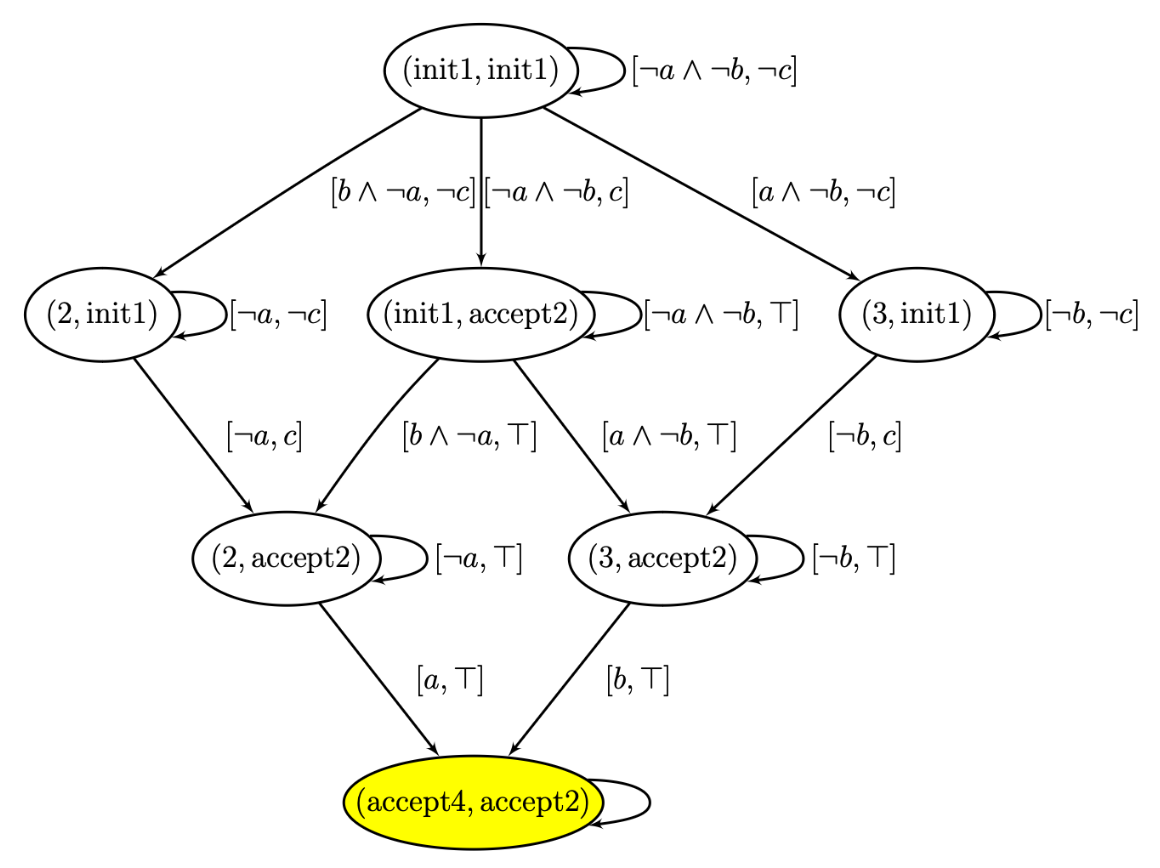}}
    \caption{The DFAs corresponding to specifications have their accepting states highlighted in yellow. In Fig.~\ref{fig:product}, only the automaton states for leaf specifications are displayed since the states of non-leaf specifications can be deduced from those of the leaf specifications in a bottom-up manner. Within the diagram, the labels inside the nodes and along the edges are derived from two parts: the first component is from $\level{1}{2}$ and the second from $\level{2}{2}$. There are four paths leading from the initial state $(\text{init}1, \text{init}1)$ to the accepting state $(\text{accept}4, \text{accept}2)$. Notably, in all these paths, the symbol $c$ is not the last one to be fulfilled.}
    \label{fig:dfas}
 \end{figure}

\begin{cexmp}{exmp:h-sc-ltl}{TPDFA}
The corresponding DFAs for each specification and the TPDFA for the hierarchical sc-LTL specified in~\eqref{eq:h-sc-ltl} are depicted in Fig.~\ref{fig:dfas}.
\end{cexmp}

\begin{defn}[Product Automaton (PA)]\label{def:product}
The product automaton combining TPDFA and TS is denoted as $\ccalP = (\ccalQ_\ccalP, \Sigma_\ccalP, \delta_\ccalP, q_{\ccalP, 0}, \ccalQ^F_\ccalP)$, where:
\begin{itemize}
\item $\ccalQ_\ccalP = \ccalS \times \ccalQ$ represents the set of product states, combining the states of TS and TPDFA;
\item $\Sigma_\ccalP = \Sigma$ is the set of symbols used in the transitions;
\item $\delta_\ccalP \subseteq \ccalQ_\ccalP \times \Sigma_\ccalP \times \ccalQ_\ccalP$ is the transition relation, as defined in Def.~\ref{def:transition_relation};
\item $\ccalL_\ccalP = \ccalL$ is the labeling function that maps states to the set of satisfied propositions;
\item $q_{\ccalP, 0} = (S_0, q_0)$ is the initial product state, combining the initial state of TS with that of TPDFA;
\item $\ccalQ_\ccalP^F = \ccalS \times \ccalQ^F$ is the set of accepting product states.
\end{itemize}
\end{defn}

Before we detail the transition relation, we present how to determine whether the observations $\ccalL_\ccalP (S)$, produced by $n$ robots, can enable transitions $q_{\ell} \rightarrow q_{\ell}' = 
(q_{\ell}^1, \ldots, q_{\ell}^m) \rightarrow (q^{1'}_{\ell}, \ldots, q^{m'}_{\ell})$ within $m$ leaf specifications where $q_\ell^i$ denotes the automaton state of the $i$-th leaf specification. This incorporates the task allocation into the edges of the product automaton. The key idea is to verify whether it is possible to construct $\sigma_{\ell} = (\subsup{\sigma}{{\ell}}{1},\ldots,\subsup{\sigma}{{\ell}}{m})$ from $\ccalL_\ccalP(S)$ in a manner that allows $(q_{\ell}, \sigma_{\ell}, q_{\ell}') \in \delta_{\ell}$ as outlined in Def.~\ref{def:total_pdfa}.

\begin{defn}[Model]\label{def:model}
Given the set of atomic propositions $\ccalL_\ccalP (S)$ generated by $n$ robots and the transition $q_{\ell} \rightarrow q_{\ell}’$ within $m$ leaf specifications, we deem $\ccalL_\ccalP (S)$ to be a {\it model} of the propositional logic formulas$\delta_{\ell}(q_{\ell}, q_{\ell}') =
\left(
\delta_{\ell}^{1}(q_{\ell}^{1}, {q_{\ell}^{1}}'),\,
\ldots,\,
\delta_{\ell}^{m}(q_{\ell}^{m}, {q_{\ell}^{m}}')
\right)$ if:
\begin{enumerate}
\item \label{cond:falsify} $\ccalL_\ccalP (S)$ does not falsify any propositional logic formula $\subsup{\delta}{\ell}{i}(\subsup{q}{\ell}{i}, {\subsup{q}{\ell}{i}}'), \quad \text{for } i \in [m]$.
\item  \label{cond:working_robots} The leaf specifications are divided into two groups $\Phi_{\ell_v}$ and $\Phi_{\ell_\epsilon}$ such that the total number of specifications $|\Phi_{\ell_v}| + |\Phi_{\ell_\epsilon}| = m$, with $|\Phi_{\ell_v}| = v$. Each group may contain any number of specifications, including none.
\item  \label{cond:executed_tasks} The set of robots is partitioned into several groups $\ccalR_1, \ldots, \ccalR_v$ and $\ccalR_\epsilon$ such that the sum of robots in these groups equals the total number of robots, $\sum_{i=1}^v |\ccalR_i| + |\ccalR_\epsilon| = n$, and each group $\ccalR_i$ and $\ccalR_\epsilon$ can include zero or multiple robots.
\item \label{cond:one-to-one} For each $i\in [v]$, there is a one-to-one correspondence between a leaf specification $\level{i}{K} \in \Phi^{K_v}$ and a group of robots $\ccalR_{i^*}$, where the atomic propositions generated by robots in $\ccalR_{i^*}$ meet the propositional logic requirement of $\level{i}{K}$, expressed as $\ccalL_{\ccalR}(S) \models \subsup{\delta}{K}{i}(\subsup{q}{K}{i}, \subsup{q}{K}{i’})$. Here, $\ccalL_{\ccalR}(S)$ represents the set of atomic propositions related to $\ccalR_{i^*}$. In this case, $\subsup{\sigma}{{\ell}}{i} = \ccalL_{\ccalR}(S)$. 

\item \label{cond:idle robot}Robots in $\ccalR_\epsilon$ are not assigned any leaf specifications, indicating they are idle.
\item \label{cond:idle task}There is no correspondence between any leaf specification in $\Phi_{\ell_\epsilon}$ and any robots, meaning that currently, no robots are engaged with the specifications in $\Phi_{\ell_\epsilon}$, but the propositional logic for each specification in $\Phi_{\ell_\epsilon}$ can be trivially fulfilled by $\varnothing$. In this case, $\subsup{\sigma}{{\ell}}{i} = \varnothing$.
\end{enumerate}
\end{defn}

Condition~\ref{cond:falsify}) requires that no leaf specification is violated by the joint robot configuration.
Condition~\ref{cond:working_robots}) categorizes robots into those actively executing tasks and those not assigned to any task. Similarly, Condition~\ref{cond:executed_tasks}) classifies tasks into those currently being performed by robots and those temporarily on hold.
Condition~\ref{cond:one-to-one}) connects robots actively executing tasks and tasks currently being performed, allowing for scenarios where multiple robots collaborate on a single task, such as jointly carrying a heavy load. Conditions~\ref{cond:idle robot}) and~\ref{cond:idle task}) pertain to situations involving idle robots and tasks that are not currently assigned to any robot, respectively.


\begin{defn}[Transition Relation]\label{def:transition_relation}
A transition from one product state $q_\ccalP = (S, q)$ to another $q’_\ccalP = (S’, q’)$ occurs if the following conditions are satisfied:
\begin{itemize}
    \item $(S, S’) \in \Delta$, as specified in Def.~\ref{def:gcs_pts};
    \item The set of propositions $\ccalL_\ccalP (S)$ is a model of the transition $\delta_{\ell}(q_{\ell}, q_{\ell}')$, as outlined in Def.~\ref{def:model}, indicating that the observed propositions at $S$ satisfy the transitions at the leaf specifications.
    \item For each level $k$ from 1 to $K$, the transition $(q_k, \sigma_k, q’_k) \in \delta_k$, as defined in Def.~\ref{def:total_pdfa}.
\end{itemize}
\end{defn}

\subsection{Prune PA}\label{Prune PA}

To manage the large size of the PA and facilitate the optimization process, we implement pruning techniques to reduce its complexity.

\begin{defn}[Essential State]
Given a pair of transitions $q_\ccalP \to q’_\ccalP$ where $q_\ccalP = (S, q)$ and $q’_\ccalP = (S’, q’)$, the product state $q_\ccalP$ is defined as an essential state if the automaton states differ, that is, if $q \not= q’$.
\end{defn}

An essential state marks that there has been progress at the task level. Based on this concept, the specifics of the pruning process are detailed in Alg.~\ref{alg: essential graph}. The procedure starts by expanding the set of essential states to include both the initial and accepting states [lines~\ref{epa:essential_states}-\ref{epa:augmentation}]. Using these essential states as the product space skeleton, we then establish connections between each pair of essential states where possible, utilizing intermediate states that navigate the robots through their configuration states according to a path determined by the RRT [lines~\ref{epa:rrt}-\ref{epa:connect}].

\begin{algorithm}[!t]
    \caption{\textbf{Construct essential PA}}
    \label{alg: essential graph}

    \KwIn{PA $\ccalP = (\ccalQ_\ccalP, \Sigma_\ccalP, \delta_\ccalP,  q_{\ccalP, 0},\ccalQ^F_\ccalP)$}
    \KwOut{Essential PA $\ccalP_e = (\ccalQ_\ccalP^e, \Sigma_\ccalP, \delta_\ccalP^e,  q_{\ccalP, 0},\ccalQ^F_\ccalP)$}

    $\ccalQ_\ccalP^* \gets \texttt{GetEssentialStates}(\ccalQ_\ccalP, \Sigma_\ccalP)$ \label{epa:essential_states}\;
    $\ccalQ_\ccalP^* \gets \ccalQ_\ccalP \cup \{q_{\ccalP, 0}\} \cup \ccalQ_\ccalP^F$ \label{epa:augmentation}\;
    $\ccalQ_\ccalP^e \gets \ccalQ_\ccalP^* $ \;
    \For{$q_\ccalP = (S, q) \in \ccalQ_\ccalP^*$}{
         \For{$q'_\ccalP = (S', q') \in \ccalQ_\ccalP^*$}{
            \If{$q \to q'$}{
                $\ccalS_{\text{connect}} \gets \texttt{IRIS\_RRT}(S, S')$ \label{epa:rrt}\;
                $\ccalQ_\ccalP^e \gets \ccalQ_\ccalP^e \cup \{q''_\ccalP = (S'', q'') \in \ccalQ_\ccalP \, | \,  S'' \in \ccalS_{\text{connect}}, q'' = q'  \}$\label{epa:connect} \;
            }
        }
    }
    $\delta_\ccalP^e = \texttt{GetTransitions}(\delta_\ccalP, \ccalQ_\ccalP^e) $ \;
    \Return $\ccalP_e = (\ccalQ_\ccalP^e, \Sigma_\ccalP, \delta_\ccalP^e,  q_{\ccalP, 0},\ccalQ^F_\ccalP)$\;
\end{algorithm}

\subsection{Optimization Formulation and Extension for Multi-robot Handover}\label{Optimization Formulation and Extension for Multi-robot Handover}

Upon building the product automaton $\ccalP$, we define a target product state $q_\ccalP^{\text{target}}$, which serves as the endpoint for all accepting product states within $\ccalQ_{\ccalP}^F$. The graph of the product automaton, denoted as $\ccalG = (\ccalV, \ccalE)$, is used to structure the optimization problem by focusing exclusively on the configuration aspect of the product states, as the automaton aspect of the product states shapes the graph structure. This approach effectively transforms the problem into motion planning in the configuration space~\cite{marcucci2023motion}. The initial configuration state is derived from the initial product state, and the target configuration state corresponds to $q_\ccalP^{\text{target}}$. This setup is tackled through an optimization problem formulated with convex programming as shown in equation~\eqref{eq:gcs}, aiming to establish a viable path from the initial to the target configuration states. In what follows, we extend this framework to accommodate multi-robot pick-and-place tasks, which include handover interactions, by introducing relevant constraints. 

Consider a scenario with $l$ objects. We define binary decision variables $b^{i, j}$ to indicate whether robot $i$ holds object $j$, with each vertex $v$ in $\ccalV$ having an associated variable $b_v^{i,j}$. When $b_v^{i, j} = 1$, it means that robot $i$ is actively transporting object $j$. In contrast, $b_v^{i, j} = 0$ indicates that the robot $i$ is not engaged in transporting the object $j$. The handover constraints governing the transfer of objects between robots are categorized into three types: incoming constraints, conflict constraints, and labeled convex set constraints. 

\subsubsection{Incoming Constraints}
Incoming constraints depict scenarios in which robots transport objects to various (intermediate) locations. For any given vertex $v \in \ccalV$, these constraints depend on whether the robot $i$, for each $i \in [n]$, is engaged in a handover with another robot. There are two possible cases for incoming constraints:

(a) {\bf Robot $i$ is not conducting the handover}. This scenario is governed by an equality constraint that ensures the continuity of possession, meaning object $j$ remains with robot $i$ during transit when no handover occurs:
\begin{align}
b_{v’}^{i, j} = b_v^{i, j}, \quad \forall v’ \in \ccalV’.
\label{eq:nonhandover}
\end{align}
Here, $v’$ refers to vertices in the set $\ccalV’$, which are predecessors leading into the vertex $v$.

It should be noted that if the optimal path in the product graph does not traverse certain vertices, the binary decision variables associated with those vertices, $b_v^{i,j}$, must be set to 0, and the incoming constraint~\eqref{eq:nonhandover} becomes irrelevant. To ensure that the incoming constraint is only applied along the optimal path, let the binary variable $y_e$ indicate whether the optimal path includes the edge from $v$ to $v'$~\cite{marcucci2024shortest}. We adopt the Big-M method to establish a connection between the optimal path and the binary decision variable $b_v^{i,j}$. The revised form of the incoming constraint~\eqref{eq:nonhandover}, when robot $i$ is not engaged in handover within the vertex $v$, is expressed as follows:
\begin{align}
M(y_e - 1) \leq (b_{v’}^{i, j} - b_v^{i, j}) \leq M(1 - y_e) , \quad \forall v’ \in \ccalV’,
\label{eq:reformed_nonhandover}
\end{align}
where $M$ is typically a large positive integer.  We set  $M = 2$  to ensure a tighter formulation.

(b) {\bf Robot $i$ is conducting the handover}. Assume $i’$, where $i’ \in [n]$, is the robot with which the handover is being conducted. The following constraint ensures that the two robots successfully transfer object $j, \forall j \in [l]$, upon reaching the location by toggling the decision variables associated with each robot at the respective vertices:
\begin{align}
b_{v’}^{i’, j} = 1 - b_v^{i, j}, \quad \forall v’ \in \ccalV’.
\label{eq:handover}
\end{align}

To ensure that the incoming constraint~\eqref{eq:handover} for robot handovers is applied strictly along the optimal path, we utilize the Big-M method to reformulate the handover constraint for robot $i$ in the vertex $v$ as follows:
\begin{align}
M(y_e - 1) \leq (b_{v’}^{i’, j} - b_v^{i, j}) \leq M(1 - y_e) , \quad \forall v’ \in \ccalV’.
\label{eq:reformed_handover}
\end{align}
The constraint~\eqref{eq:reformed_handover} effectively prevents a handover from occurring unless the path is optimal.

\subsubsection{Conflict Constraints}
To ensure that each robot handles no more than one object at a time, and each object is managed by only one robot simultaneously, the following constraints apply:
\begin{subequations}
\begin{align}
\sum_{j \in [l]} b_{v}^{i,j} &\leq 1, \quad \forall v \in \ccalV, \forall i \in [n], \\
\sum_{i \in [n]} b_{v}^{i,j} &\leq 1, \quad \forall v \in \ccalV, \forall j \in [l].
\end{align}
\end{subequations}

\subsubsection{Labeled Convex Set Constraints}
For a product automaton $\ccalP$ associated with labeled convex set that is labeled with a robot, denote by $i$, whose state aligns with the location of object $j$ and the vertex $v \in \ccalV'$:
\begin{align}
    b_v^{i, j} = 1.
\end{align}

Those handover constraints ensure an orderly and conflict-free transfer of objects among the robots. Due to the involvement of binary decision variables in the handover constraints, we solve the optimization Problem~\ref{problem1} using mixed-integer convex programming (MICP).
\section{Theoretical Analysis}\label{Theoretical Analysis}

\begin{theorem}[Soundness]
  The returned path $p$ satisfies the hierarchical sc-LTL specifications $\Phi$.
\end{theorem}

\begin{proof}
    The proof consists of two steps. In the first step, we construct a state-specification sequence $\tau$ as in Def.~\ref{defn:state_specification} from the path $p$. In the second step, we prove that the state-specification sequence $\tau$ satisfies the hierarchical sc-LTL $\Phi$. 

  To obtain the state-specification sequence $\tau$, the goal is to pair each robot with a leaf specification that it is undertaking, if any.  Note that each point in the path $p$ is a product state $q$ composed of a configuration state $s = (s_1, \ldots, s_n)$ and a product DFA state $q_\ccalP = (q_K, \ldots, q_1)$ with $q_{\ell} = (\subsup{q}{{\ell}}{1}, \ldots, \subsup{q}{{\ell}}{m})$ for leaf specifications. Let $q'_\ccalP = (q'_K, \ldots, q'_1)$ with $q'_{\ell} = (\subsup{q}{{\ell}}{1'}, \ldots, \subsup{q}{{\ell}}{m'})$ denote the next state of $q_\ccalP$ in the path $p$. As stated in Def.~\ref{def:transition_relation}, $\ccalL_\ccalP(S)$ is a model of $\delta_\ell(q_\ell, q'_\ell)$.  Next, we analyze depending on conditions in Def.~\ref{def:model}. Specifically, condition~\ref{cond:falsify}) ensures that no leaf specification is violated. Furthermore, we pair the leaf specification $\level{i}{\ell} \in \Phi_{\ell_v} $  with every robot in the corresponding set $\ccalR_{i^*}$, according to condition~\ref{cond:one-to-one}), otherwise, we pair robot $i \in \ccalR_{\epsilon}$ with null specification $\epsilon$ according to condition~\ref{cond:idle robot}).
   
  Given the state-specification sequence $\tau$,  the labels generated by the sequence of configuration states satisfy not only the leaf specifications, but also the non-leaf specifications, according to the transition relation in Def.~\ref{def:transition_relation},. Moreover, it reaches an accepting product state where the top-most specification $\level{1}{1}$ is satisfied, implying that the hierarchical sc-LTL is satisfied according to semantics in~\cite{luo2024simultaneous}.
\end{proof}

\begin{definition}[Incompatible specifications]
    Two sc-LTL specifications are considered incompatible if there exists a path that satisfies one but inevitably violates the other, regardless of how the path is extended.
\end{definition}

For instance, $\phi_1 = \Diamond a$ and $\phi_2 = \neg a\ \ccalU\ b$ are incompatible, as a path that produces label $a$ but not label $b$ satisfies $\phi_1$ while violating $\phi_2$. This path cannot be extended to satisfy $\phi_2$. We say that a hierarchical sc-LTL specification  $\Phi$  is considered compatible if it does not include any pair of leaf specifications that are mutually incompatible.

\begin{theorem}[Completeness]
Assuming the hierarchical sc-LTL $\Phi$ is compatible, up to the space decomposition and trajectory parameterization, our approach returns a path that satisfies $\Phi$.
\end{theorem}

\begin{proof}
The CS-based product transition system (Def.~\ref{def:gcs_pts}) encompasses all possible behaviors of the robot system, while the total product of DFAs encompasses all solutions to fulfill the hierarchical sc-LTL (Def.~\ref{def:total_pdfa}).  The construction of PA (Def.~\ref{def:product}) is based on the concept of a model (Def.~\ref{def:model}). In particular, condition~\ref{cond:falsify}) excludes propositional logic formulas that contradict each other. Since we consider only compatible hierarchical sc-LTL specifications, no contradictory propositional logic formulas arise. Conditions~\ref{cond:working_robots})-\ref{cond:idle task}) ensure the existence of a feasible task allocation. Consequently, the product system (Def.~\ref{def:product}) includes all behaviors of the robot system that conform to the hierarchical sc-LTL.
According to~\cite{marcucci2023motion}, by increasing the number of convex sets and the degree of B\'ezier curves to enhance the approximation, the MICP is guaranteed to find a path.
\end{proof}
    
\section{Experiments}
\label{Experiments}
We evaluate our approach across a variety of multi-robot task scenarios, including planar robot motion planning, coordination among multiple robotic manipulators with handovers, quadrupedal mobile robots performing manipulator handovers, and a structured industrial environment featuring several robots and a conveyor system. All experiments were conducted using the Drake \cite{drake}, and executed on a desktop computer equipped with an Intel i9 processor and 32GB of RAM. To solve Mixed-Integer Convex Programming (MICP), we use the MOSEK solver \cite{MOSEK} via the Drake interface. Our open-source code can be accessed at \href{https://github.com/intelligent-control-lab/Task_Motion_Planning_with_HLTL_and_GCS.git}{https://github.com/intelligent-control-lab/Task\_Motion\_Planning\_with\_HLTL\_and\_GCS.git}. The running times for the algorithm of all scenarios are detailed in \Cref{tb:Algorithm Running Time.}. We precompute the CS-based transition system offline and report computation times for online modules, including PA construction, PA pruning, and MICP solving.  The demonstration video is available at \href{https://youtu.be/FPTLGm5iigc}{this link}.

\renewcommand{\arraystretch}{1.2} 
 \begin{table*}[!th]
 \centering
 \captionsetup{width=15cm}
\caption{Algorithm Running Time.}
\label{tb:Algorithm Running Time.}
\begin{tabular}{lcccccc}
\bhline
 Tasks & Figure & Construct PA ($s$)& Prune PA ($s$) & Solver Time ($s$)\\\hline
two-robot motion planning &  ~\ref{fig: two-robot-case}(a) & 0.106 & 0.002 & 5.531\\
two-robot handover &  ~\ref{fig: two-robot-case}(b)  & 0.061 & 0.001 & 0.277\\
four-robot handover (scenario 1) & \ref{fig: four robots handover} & 12.711 & 0.007 & 11.601\\
four-robot handover (scenario 2) & \ref{fig: four-robot-handover2} & 7.262 & 0.003 & 7.691\\
four-robot handover with obstacle (scenario 3) & \ref{fig: four-robot-handover3} & 8.355 & 0.003 & 6.291\\
four-robot handover (scenario 4) & \ref{fig: four-robot-handover4} & 187.18 & 0.083 & 19.10\\

Spot-robot handover &  \ref{fig: spot_handover} & 0.074 & 0.001 & 0.378 \\
two-robots with conveyor &  \ref{fig: two_iiwa_conveyor} & 5.738 & 0.004 & 1.019\\
\bhline
\end{tabular}
\end{table*}

\subsection{Planar Motion Planning Case}

The planar motion planning scenario depicted in Fig.~\ref{fig:key_door} involves a robot tasked with collecting five keys to navigate through the corresponding doors. This benchmark example, noted for its complex specifications as proposed in \cite{kurtz2023temporal}, originally required transforming the sc-LTL formula into DFA. The conversion process from sc-LTL to DFA is known to exhibit double-exponential complexity \cite{belta2017formal}, leading to a prolonged conversion time. By employing hierarchical sc-LTL, our approach significantly reduces this complexity. The original standard sc-LTL formula is
\begin{equation}
\begin{aligned}
    \phi = \bigwedge_{i=1}^5 \neg \mathsf{door}_i \, \mathcal{U} \, \mathsf{key}_i \wedge \Diamond \mathsf{goal}.
\end{aligned}
\end{equation}
The hierarchical sc-LTL specifications are represented as 
\begin{equation}
\begin{aligned}
     L_1: \quad & \level{1}{1} = \bigwedge_{i=1}^5 ( \Diamond \level{i}{2}) \wedge \Diamond \level{6}{2} \\
    L_2: \quad &  \level{i}{2} = \neg \mathsf{door}_i \, \mathcal{U} \, \mathsf{key}_i , \quad i = 1, \ldots, 5.\\
    \quad & \level{6}{2} = \Diamond \mathsf{goal}.
\end{aligned}
\end{equation}

\begin{figure}[!th]
    \centering
    \includegraphics[width=0.5\linewidth,  trim={5.7cm 2.5cm 5.1cm 2.7cm}, clip]{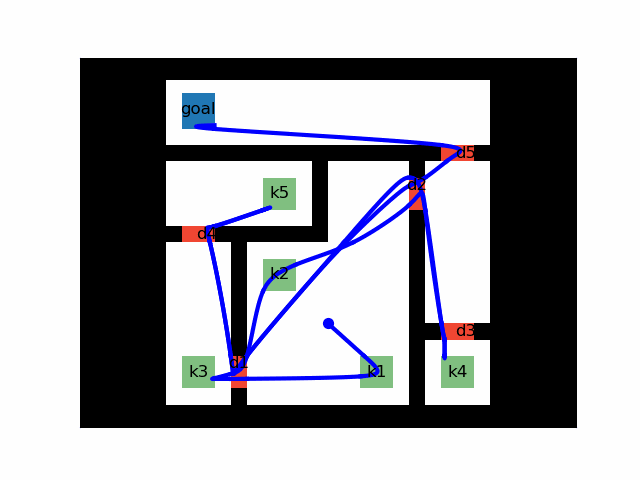}
    \caption{The door puzzle problem, where the blue dot represents the initial robot location.~\cite{kurtz2023temporal,vega2018admissible}.}
    \label{fig:key_door}
\end{figure}

Hierarchical sc-LTL enhances the efficiency of representing temporal logic specifications, resulting in quicker conversion times and faster motion planning. This improved efficiency is evident in the performance comparison shown in Table~\ref{tab:comparison_with_ltl_gcs}. Our approach yields a solution with a cost almost identical to that of the method described in \cite{kurtz2023temporal}, yet it achieves this result approximately 40 times faster. 

\renewcommand{\arraystretch}{1.3} 
 \begin{table}[!th]
     \centering\footnotesize
      \begin{tabular}{cccccc}
      \bhline
      \multirow{2}{*}{Method} & \multicolumn{4}{c}{Time ($s$)} & \multirow{2}{*}{Cost} \\
       \cline{2-5}
       & Sc-LTL to DFA & Construct PA &   Solver  &  Total &\\
       \hline
      \cite{kurtz2023temporal} & 401.6 & 94.4 &  2.6 & 498.7  & 774.7\\
      Ours &  {\bf 8.0} &  {\bf 4.0} &  {\bf 0.9}  &  {\bf 13.0} &  774.2 \\
     \bhline
     \end{tabular}
     \caption{Performance comparison.}
     \label{tab:comparison_with_ltl_gcs}
     \vspace{-10pt}
 \end{table}

\subsection{Multi-robot Motion Planning and Handover Case}
To illustrate the scalability of our method in high-dimensional spaces, we evaluated it in four systems with different tasks: a system involving two robotic manipulators (Fig.~\ref{fig: two-robot-case}), a system with four robotic manipulators with multiple objects(Figs.~\ref{fig: four-robot-handover2}-\ref{fig: four-robot-handover4}), a system with robotic manipulators and mobile robots (Fig.~\ref{fig: spot_handover}), and a structured industrial environments with robotic manipulators and conveyor (Fig.~\ref{fig: two_iiwa_conveyor}).

\subsubsection{Two robotic manipulators}
In the initial example, two robotic manipulators are tasked to pick up an object from target 1 and place it on target 2. The hierarchical sc-LTL specifications for this scenario are articulated as follows:
\begin{equation}
\begin{aligned}
L_1: \quad & \level{1}{1} = \Diamond \level{1}{2}\\
L_2: \quad & \level{1}{2} = \Diamond (\mathsf{target1} \wedge \Diamond \mathsf{target2}).
\end{aligned}
\end{equation}

Our planning algorithm uses B\'ezier splines to navigate a valid path, ensuring $C^2$ continuity for smooth trajectories. An L2 norm for the length of the joint path is also integrated to optimize for the shortest possible route. We evaluate our planner in two distinct pick-and-place scenarios. In the first scenario, depicted in Fig.~\ref{fig:scenario_a}, where targets 1 and 2 are equidistant from both robots, our planner decides only the robot with the shortest path to execute the task, leaving the other robot stationary to minimize the total path length. In the second scenario, illustrated in Fig.~\ref{fig:scenario_b}, each robot exclusively accesses one of the targets, necessitating a handover to complete the task. Consequently, one robot picks up the object and passes it to the other, which then places it at target 2. These tests confirm that our planner adeptly identifies the most efficient strategy autonomously, eliminating the need for pre-defined orders on robot movement or handover timing.

\begin{figure}[!th]
    \centering
       \subfigure[Both robots can reach the target positions 1 and 2.]{
      \label{fig:scenario_a}
      \includegraphics[width=\linewidth]{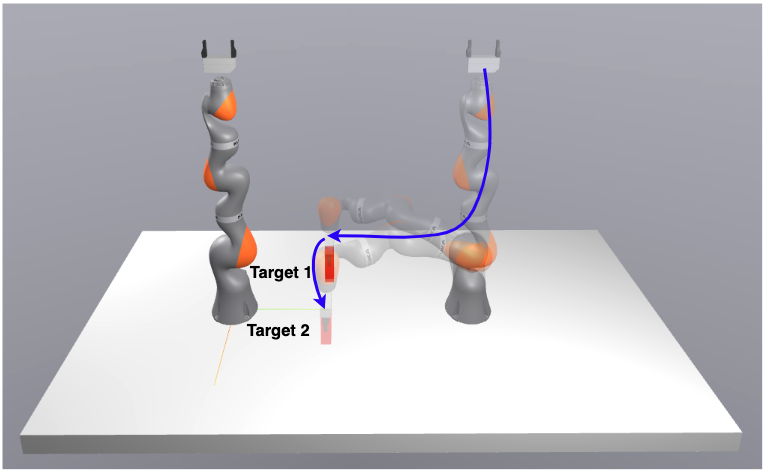}}
     \subfigure[The left robot can only reach target 1, and the right robot can only reach target 2. The handover is necessary to complete the task.]{
      \label{fig:scenario_b}
      \includegraphics[width=\linewidth]{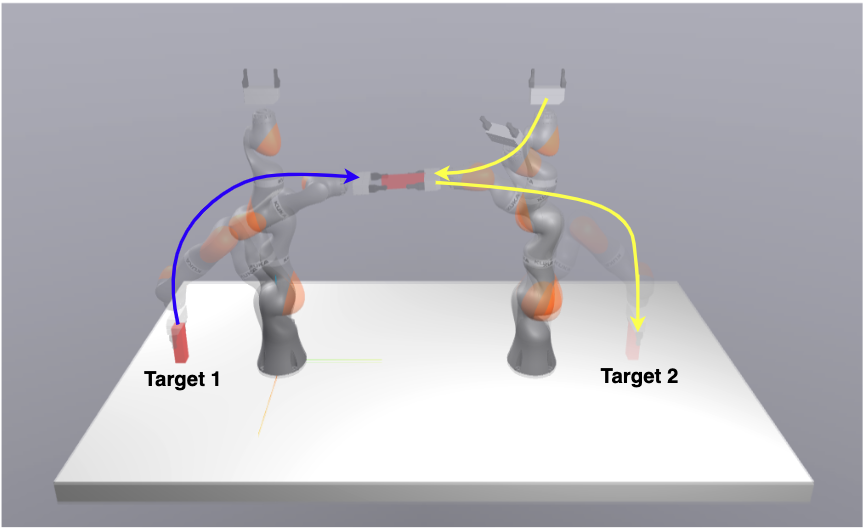}}
    \caption{Example setup with two robotic manipulators with one movable object.}  
    \label{fig: two-robot-case}
    \vspace{-10pt}
\end{figure}

\subsubsection{Four robotic manipulators with multiple objects}\label{subsubsec:four robot}
The second example shows our planner in a more complex scenario and task specification that involves four KUKA iiwa robotic manipulations, with a 28 degree of freedom (DOF). We evaluated our algorithm in several distinct scenarios to assess its performance under different task specifications.
\paragraph{Scenario 1}
\label{Scenario 1}
In the first scenario, the robot was aligned in a linear arrangement, illustrated in Fig.~\ref{fig: four robots handover}. The robots are assigned to pick up three objects from specifically labeled pickup positions (targets 1, 2, and 3) and place them in corresponding placement positions (targets 4, 5, and 6). The hierarchical sc-LTL specification for this task is given by:
\begin{equation}
\begin{aligned}
\label{eq:four-robots HLTL}
L_1: \quad &  \level{1}{1} = \Diamond (\level{1}{2} \wedge \Diamond (\level{2}{2} \wedge \Diamond \level{3}{2}))\\
L_2: \quad & \level{1}{2} = \Diamond (\mathsf{target1} \wedge \Diamond \mathsf{target4}) \\
& \level{2}{2} = \Diamond (\mathsf{target2} \wedge \Diamond \mathsf{target5}) \\
& \level{3}{2} = \Diamond (\mathsf{target3} \wedge \Diamond \mathsf{target6}),
\end{aligned}
\end{equation}
which states that the robot should pick up movable objects from targets regions 1, 2, and 3, and place them in regions 4, 5, and 6 in order.  
In this example, the planner does not take into account the interactions between the gripper and objects, nor does it consider the dynamics constraints of the robot trajectory. Instead, the focus of our planner is on identifying a collision-free path in the configuration space that satisfies the hierarchical sc-LTL specifications within a 28-DOF system. This is a benchmark example proposed in~\cite{envall2023differentiable}. The problem was solved by nonlinear trajectory optimization with smoothed discrete variables, solving multi-robot handover with one object in 36.5 seconds. However, this gradient-based method relies on local information and often faces challenges in finding feasible trajectories that deviate significantly from the initial guess, leading to suboptimal solutions with unnecessary joint movements. In contrast, the considerably more complex specification with three objects handover is solved by our proposed method in approximately 11.6 seconds. In addition, our method provides a sound, complete, and lower-cost solution without relying on an initial guess of the solver.

\paragraph{Scenario 2}  
\label{Scenario 2}
In the second scenario, we designed a different layout from Scenario 1, with the robots arranged in a narrow rectangular formation. The hierarchical sc-LTL specifications are still~\eqref{eq:four-robots HLTL}. This case is more complex than the previous one due to the narrow robot workspace, requiring the planner to ensure collision avoidance and determine the necessity of handovers between robots. The result is illustrated in Fig.~\ref{fig: four-robot-handover2}. Note that in Fig.~\ref{fig: four-robot-handover2}(b) and (c), while two robotic manipulators are transferring the yellow and red objects, the bottom robot simultaneously moves to transfer the blue object, effectively minimizing the overall completion time.
\begin{figure}[!th]
    \centering
    \includegraphics[width=1\linewidth]{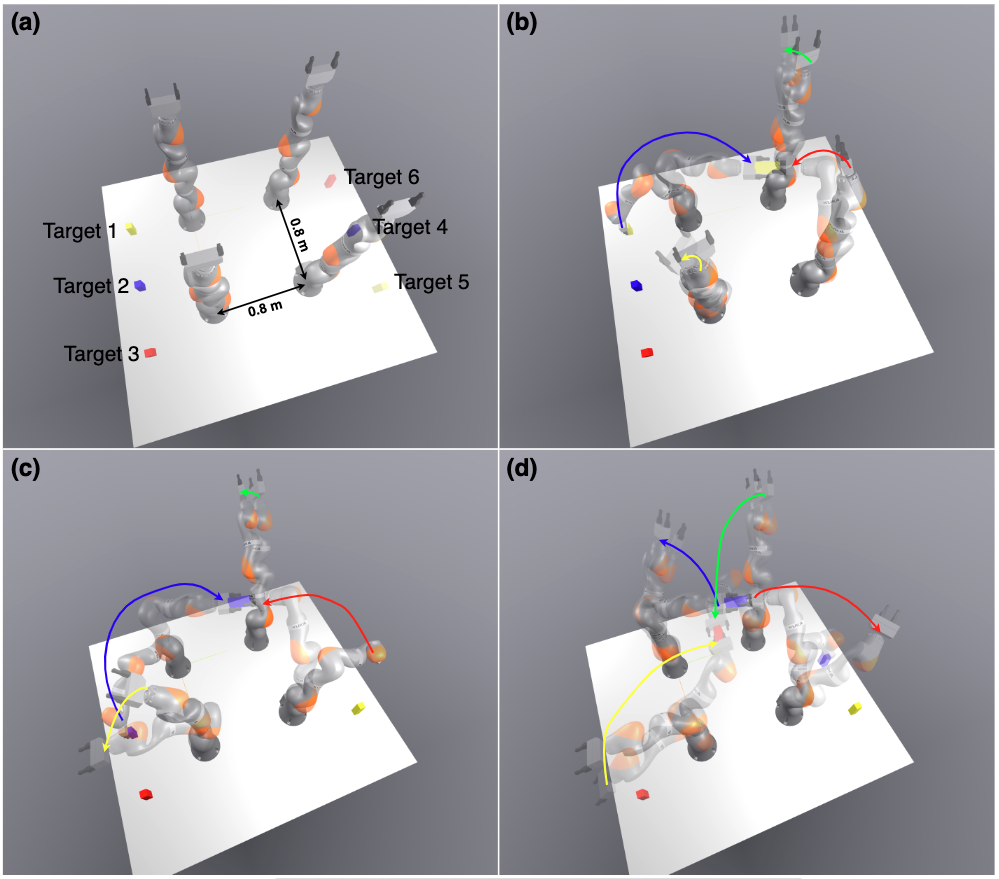}
    \caption{Setup involving four robotic manipulators operating in a rectangular formation. The robots are tasked with handling three movable objects.}
    \label{fig: four-robot-handover2}
    \vspace{-10pt}
\end{figure}

\paragraph{Scenario 3}  
In this scenario, the layout is the same as Scenario 2, with the robots arranged in a narrow rectangular formation. We introduce a conditional passable region in the middle, as illustrated in Fig.~\ref{fig: four-robot-handover3}. The task involves picking up two objects from designated pickup positions (targets 1 and 2), placing them at the corresponding placement positions (targets 3 and 4), and returning the object from target 3 to target 1 while avoiding the conditional passable region. The hierarchical sc-LTL specifications are written as:
\begin{equation}
\begin{aligned}
\label{eq:four-robots HLTL 1}
L_1: \quad &  \level{1}{1} = \Diamond (\level{1}{2} \wedge \Diamond (\level{2}{2} \wedge \Diamond \level{3}{2}))\\
L_2: \quad & \level{1}{2} = \Diamond (\mathsf{target1} \wedge \Diamond \mathsf{target3}) \\
& \level{2}{2} = \Diamond (\mathsf{target2} \wedge \Diamond \mathsf{target4}) \\
& \level{3}{2} = \Diamond (\mathsf{target3} \wedge \neg \mathsf{obstacle} ~\mathcal{U}~ \mathsf{target1}).
\end{aligned}
\end{equation}
Note that in Figs.~\ref{fig: four-robot-handover3}(b) and (c), two robots perform the handover of the yellow and blue objects within the passable regions to minimize the trajectory length. However, robots avoid the passable region when transferring the yellow object back, as shown in Fig.~\ref{fig: four-robot-handover3}(d).
\begin{figure}[!th]
    \centering
    \includegraphics[width=1\linewidth]{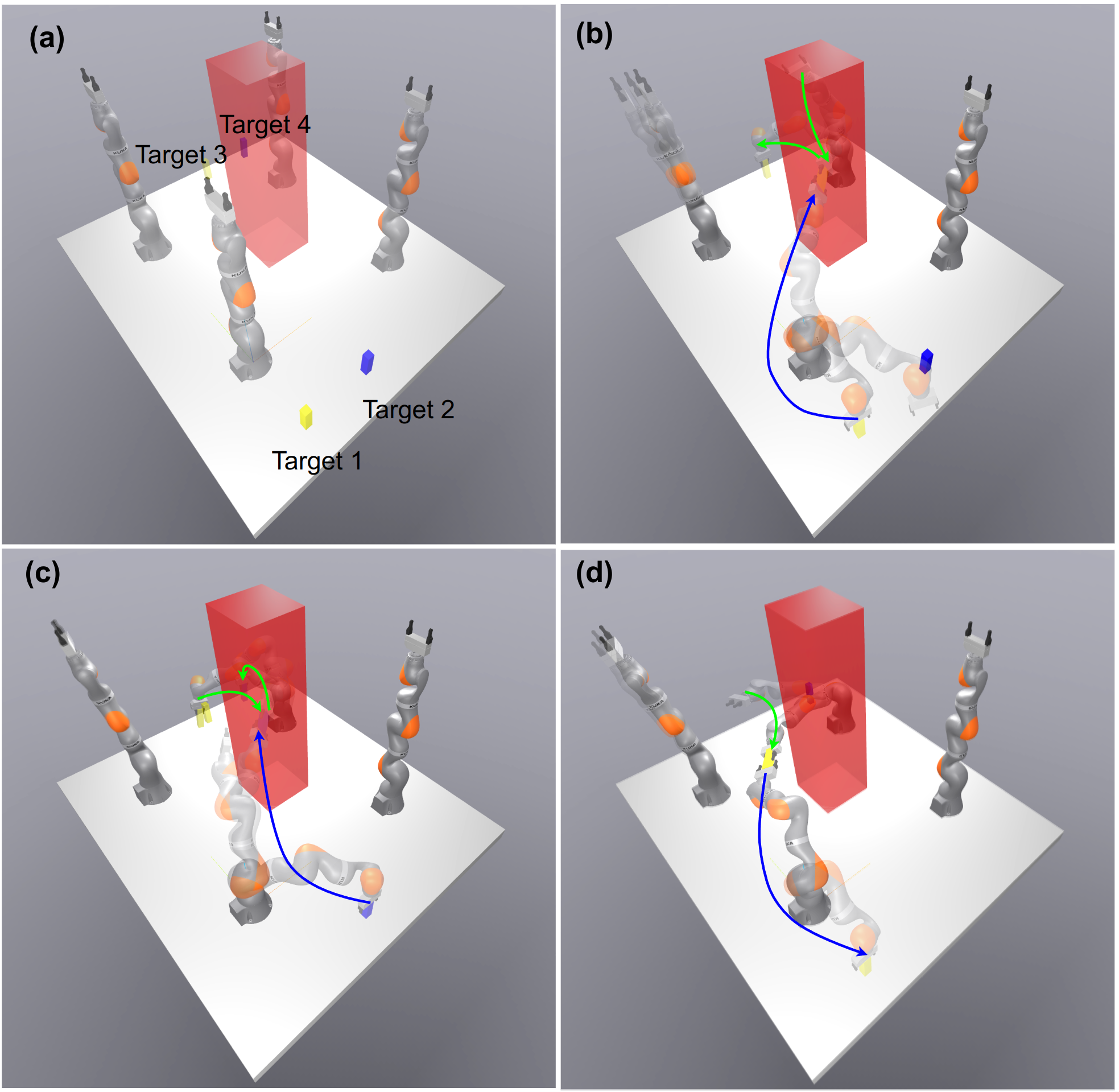}
    \caption{Setup involving four robotic manipulators operating in a rectangular formation with a passable region (red).}
    \label{fig: four-robot-handover3}
    \vspace{-10pt}
\end{figure}
\paragraph{Scenario 4}  
In the scenario, as illustrated in Fig.~\ref{fig: four-robot-handover4}, we consider more complex tasks with 3 levels of hierarchical sc-LTL. The robots are assigned two tasks involving picking up objects from one location to another. For those tasks, the robot is allowed to execute either one of them. The hierarchical sc-LTL specifications for this task are as follows:
\begin{equation}
\begin{aligned}
\allowdisplaybreaks
\label{eq:four-robots HLTL 5}
L_1: \quad &  \level{1}{1} = \Diamond (\level{1}{2} \lor \level{2}{2}) \\
L_2: \quad &  \level{1}{2} = \Diamond \level{1}{3} \wedge \Diamond \level{2}{3} \\
& \level{2}{2} = \Diamond (\level{3}{3} \wedge \Diamond \level{4}{3} )\\
L_3: \quad & \level{1}{3} = \Diamond (\mathsf{target1} \wedge \Diamond \mathsf{target5}) \\
& \level{2}{3} = \Diamond (\mathsf{target2} \wedge \Diamond (\mathsf{target4} \lor \mathsf{target6})) \\
& \level{3}{3} = \Diamond (\mathsf{target1} \wedge \neg \mathsf{obstacle} ~\mathcal{U}~ \mathsf{target5}) \\
&\level{4}{3} = \Diamond (\mathsf{target2} \wedge \Diamond (\mathsf{target1} \lor \mathsf{target5}))
\end{aligned}
\end{equation}

The solution prioritizes completing task $\level{2}{2}$, as it has a lower trajectory length cost compared to task $\level{1}{2}$. In addition, in Fig.~\ref{fig: four-robot-handover4}(c), the robot places the blue object from target 2 to target 1 instead of target 5 to minimize the trajectory length.

\begin{figure}[!th]
    \centering
    \includegraphics[width=1\linewidth]{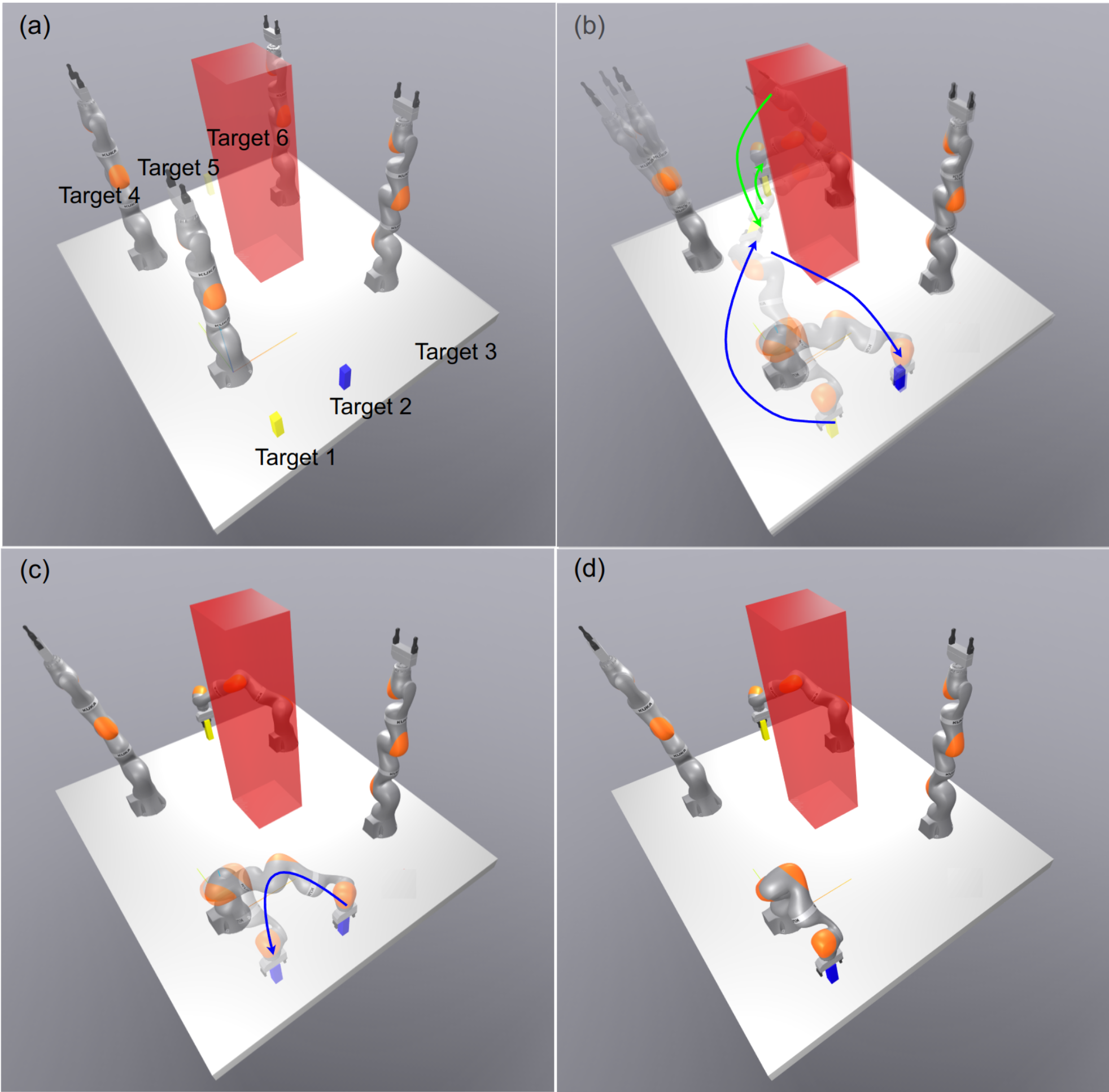}
    \caption{Setup involving four robotic manipulators operating in a rectangular formation with a passable region (red).}
    \label{fig: four-robot-handover4}
    \vspace{-10pt}
\end{figure}
\subsubsection{Mobile robots and robotic manipulators}
In prior experiments, our focus was on stationary handovers between manipulators. In this example, we explore handovers involving mobile robots and robotic manipulators, using a Boston Dynamics Spot robot and a Kuka iiwa robot. The task is to move the objects from target 1 to target 2, as depicted in Fig.~\ref{fig: spot_handover}. The hierarchical sc-LTL specification are defined as follows:
\begin{equation}
\begin{aligned}
L_1: \quad & \level{1}{1} = \Diamond \level{1}{2}\\
L_2: \quad & \level{1}{2} = \Diamond (\mathsf{target1} \wedge \Diamond \mathsf{target2}).
\end{aligned}
\end{equation}
Given that the iiwa robot can access target 1 but not target 2, a handover to the Spot robot is necessary. We assume that the Spot robot has a floating base with 3 DoF for movement and a 6-DoF manipulator for handling tasks. The goal of planning is to identify a collision-free trajectory that minimizes both the path length (L2 norm) and the total time. Our proposed method generated an optimal trajectory in just 0.378 seconds. In this trajectory, the Spot robot operates at maximum speed and performs a continuous handover with the iiwa robot, ensuring a transition of the object without any stop.

\begin{figure}[!th]
    \centering
    \includegraphics[width=1\linewidth]{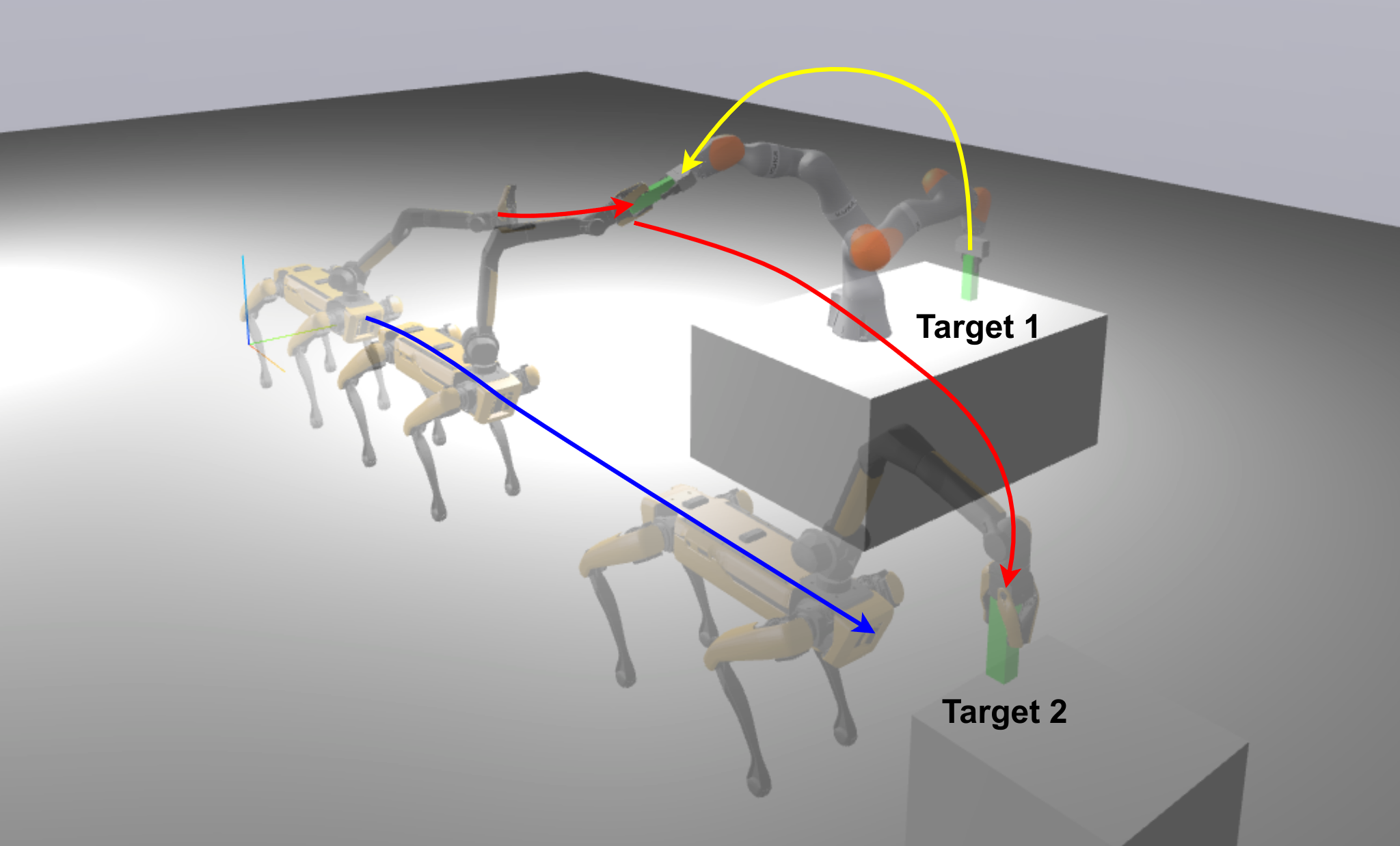}
    \caption{Example setup featuring a mobile robot paired with manipulators tasked with handling a single movable object.}
    \label{fig: spot_handover}
    \vspace{-10pt}
\end{figure}

\subsubsection{Structured industrial environments with robotic manipulators and conveyor}
Lastly, we construct a factory setting where two robotic manipulators are integrated with a conveyor system, as illustrated in Fig.~\ref{fig: two_iiwa_conveyor}. This example is designed to demonstrate the capability of our planner in structured industrial environments. The task involves transporting three objects using robotic manipulators and the conveyor system, moving them from targets 1, 2, and 3 to targets 4, 5, and 6. The hierarchical LTL specifications follow the formulation in~\eqref{eq:four-robots HLTL}. Each robot has 7 DoF, and the conveyor has one DoF. The objective of our planner is to generate a collision-free trajectory in 15 DOF that minimizes the L2 norm of the path length and the cycle times across all robots and conveyor. Our planner successfully finds a time-optimal path in just 1.019 seconds.

\begin{figure}[!t]
    \centering
    \includegraphics[width=1\linewidth]{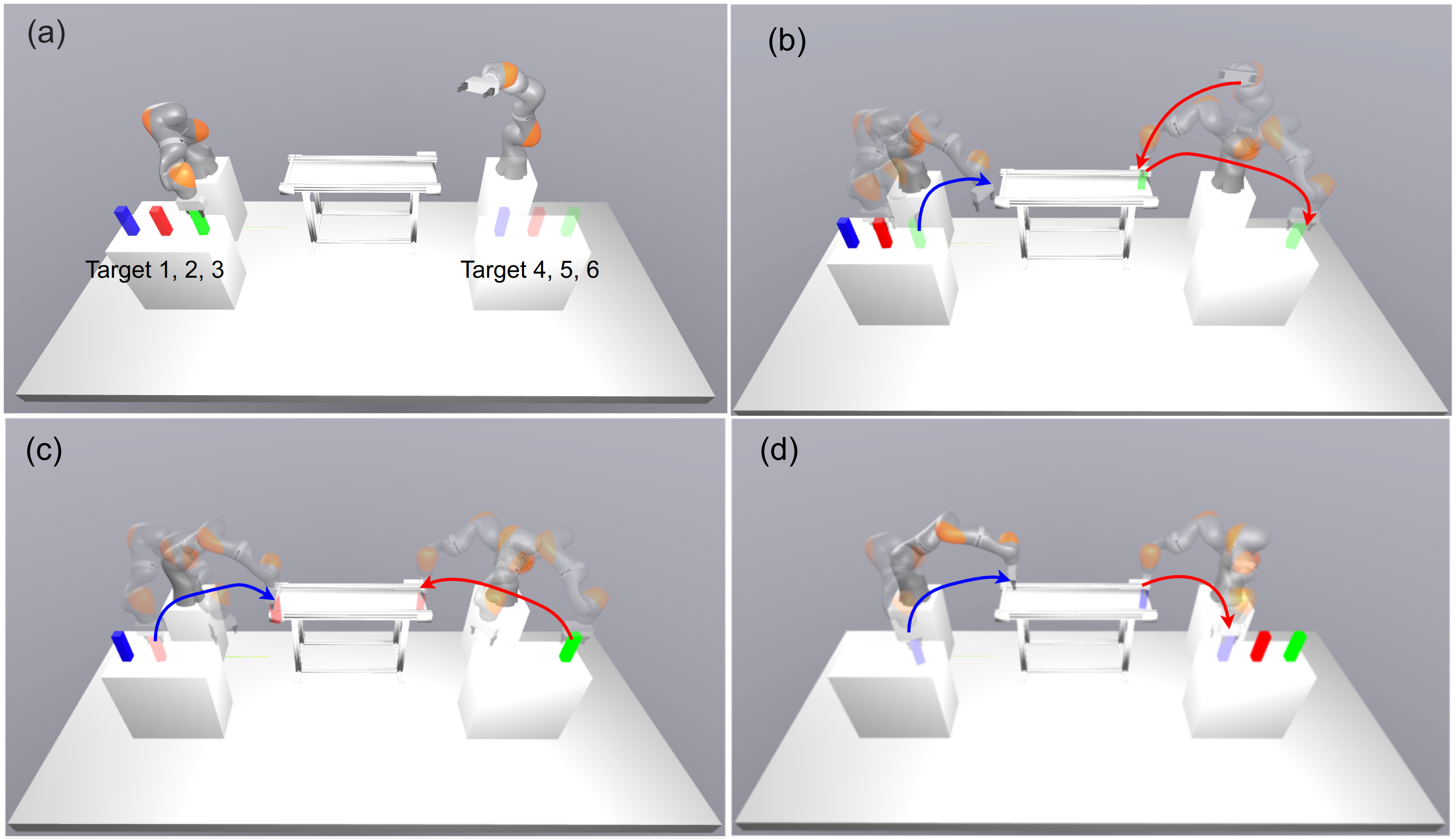}
    \caption{Example setup featuring two robotic manipulators and one conveyor with three movable objects.}
    \label{fig: two_iiwa_conveyor}
    \vspace{-10pt}
\end{figure}

\subsection{Robot Experiments}
We demonstrate the deployment of our planner on real robot hardware, using four WidowX 200 robots equipped with two-finger grippers. The hierarchical sc-LTL specifications follow the formulation~\eqref{eq:four-robots HLTL}. Our planner successfully generates collision-free trajectories, which are then executed via the ROS-based position controller provided by Tossen Robotics~\cite{interbotix_ros_manipulators}. The experimental result is illustrated in Fig.~\ref{fig: four_arms}.

\begin{figure}[!th]
    \centering
    \includegraphics[width=1\linewidth]{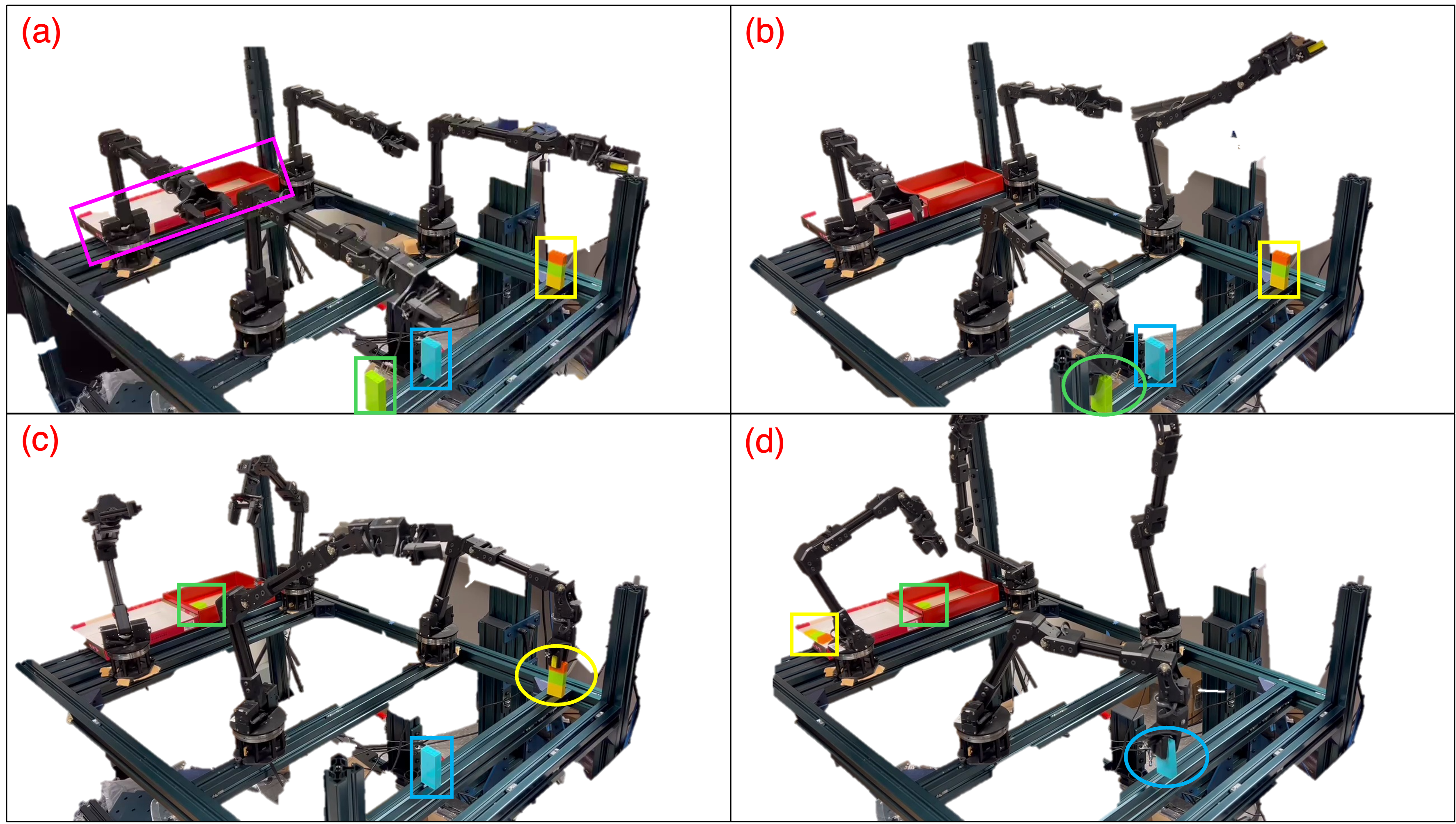}
    \caption{A setup with four manipulators designed for pick-and-place tasks.}
    \label{fig: four_arms}
    \vspace{-10pt}
\end{figure}

\subsection{Scalability}
We assess the scalability of our proposed method by analyzing solve times as the number of tasks increases, using the four-robot manipulator setup described in Scenario 2. Each task involves picking and placing a single object. The source and target locations of objects are carefully chosen to avoid the need for handovers. Tasks are evaluated in two settings: with and without handover constraints. Since handover constraints introduce binary decision variables, the problem is solved using mixed-integer convex optimization. In the absence of handover constraints, standard convex optimization is sufficient. The results in Fig.~\ref{fig: Scalability Analysis} indicate that solve times increase approximately linearly with the number of tasks when handover constraints are enforced. In contrast, without handover constraints, solve times scale more efficiently, exhibiting sublinear performance.


\begin{figure}[!th]
    \centering
    \includegraphics[width=1\linewidth]{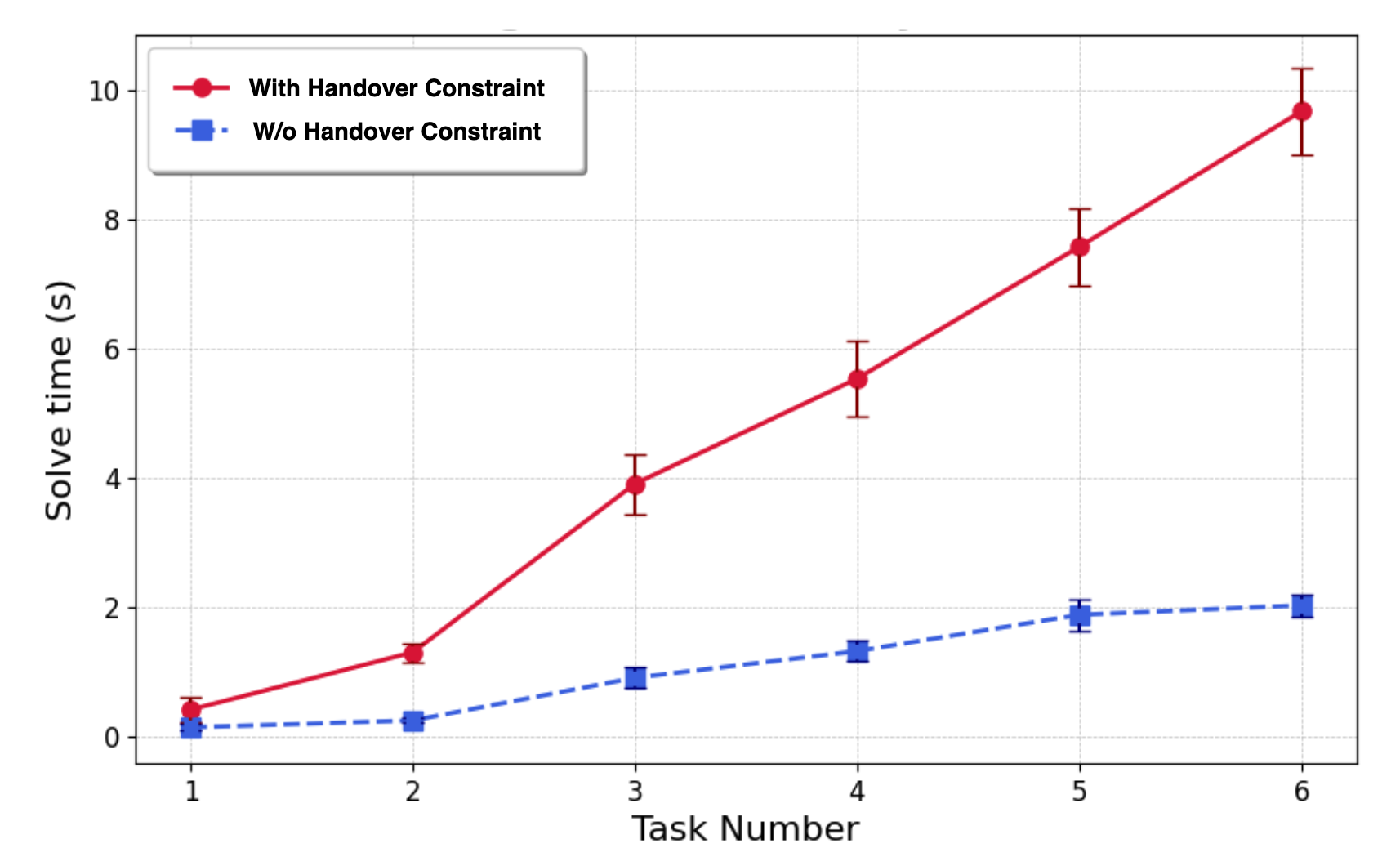}
    \caption{Scalability results of runtimes w.r.t number of tasks.}
    \label{fig: Scalability Analysis}
    \vspace{-10pt}
\end{figure}
\section{Limitations}
\label{Limitations}
In this section, we briefly outline the limitations of our proposed approach. A major drawback of our method is that the complexity of MICP scales exponentially with the number of binary variables~\cite{Mixed-Integer-Encoding}. The MICP solution for the original shortest path problem in a graph of convex sets has a very tight convex relaxation. Therefore, this problem can often be solved globally optimally with convex optimization and rounding. However, to handle handover scenarios, we introduced additional integer variables, which loosen the convex relaxation. As a result, our proposed method uses MICP to solve the hierarchical temporal logic TAMP problem with handover constraints. With the increasing complexity of the hierarchical sc-LTL specification and the number of convex sets in the transition system, the MICP scales poorly. However, in practice, this does not appear to be a significant concern. After simplifying the product graph with our graph prune method, the optimal trajectory for a 28-DOF, four-robot system, involving 719 integer decision variables, can be computed in under 10 seconds.
 
Another limitation is that we use the IRIS-NP algorithm \cite{petersen2023growing} to generate collision-free space convex sets. The algorithm provides iterative procedures for inflating convex regions of free space. However, the IRIS-NP algorithm only has probabilistically certified that the convex region is collision-free, and it may take a long time to generate the region. In addition, the generated convex region will change when the objects transfer from one robot to another. Efficiently generating certified collision-free convex regions in the configuration space and quickly adapting the region to changes in collision geometry are important areas for future research.  
 
 
\section{Conclusions}
\label{Conclusions}
We address the multi-robot task and motion planning (TAMP) problem subject to hierarchical temporal logic specifications. The key idea is to convert the optimal planning problem into a shortest path problem in a product graph, which is then solved using MICP. Unlike methods based on nonlinear trajectory optimization, sampling-based search, or learning-based approaches, our approach can handle non-convex multi-robot motion planning challenges, delivering solutions with sound and complete results within reasonable solve times.
\section*{Acknowledgments}
The authors would like to thank Prof. Oliver Kroemer for providing Trossen robots, as well as to anonymous reviewers for their insightful feedback.
\bibliographystyle{unsrtnat}
\bibliography{ref}

\end{document}